\documentclass[twoside]{article}

\usepackage[dvipsnames]{xcolor}

\usepackage[tbtags]{amsmath}
\usepackage{amsthm}
\usepackage{bm}
\allowdisplaybreaks
\usepackage{amssymb,mathrsfs}
\usepackage{amsfonts}
\usepackage{upgreek}
 \usepackage{xargs}
 \usepackage{graphicx}
 
\usepackage{multirow}

\usepackage{paralist}
\usepackage{booktabs}       
\usepackage{nicefrac}       
\usepackage{multirow}
\usepackage{pifont}

\newcommand{\xmark}{\textcolor{red}{\ding{55}}}

\newtheorem{theorem}{Theorem}
\newtheorem{lemma}{Lemma}

\usepackage[colorlinks = true, citecolor = blue, linkcolor = red]{hyperref}
\usepackage{cleveref}
\usepackage{autonum}

\newcommand{\greencheck}{{\color{OliveGreen}\checkmark}}
\newcommand{\redcross}{{\color{BrickRed}\xmark}}

\newtheorem{assumption}{\textbf{H}\hspace{-3pt}}
\crefformat{assumption}{{\textbf{H}}#2#1#3}

\usepackage{enumitem}

\usepackage[round,sort&compress]{natbib}

\usepackage{xspace}
\usepackage{soul}
\newcommand{\ours}{\textcolor{magenta}{\texttt{FLIC}}\xspace}

\definecolor{darkgreen}{RGB}{0,128,0}
\definecolor{darkorange}{RGB}{255,140,0}
\definecolor{darkblue}{RGB}{0,0,139}

\def\nset{\mathbb{N}}
\def\nsets{\mathbb{N}^*}

\def\rmd{\mathrm{d}}

\newcommand{\LeftEqNo}{\let\veqno\@@leqno}

\newcommand{\floor}[1]{\left\lfloor #1 \right\rfloor}

\newcommand{\parentheseLigne}[1]{(#1 )}

\newcommandx\sequence[3][2=,3=]
{\ifthenelse{\equal{#3}{}}{\ensuremath{( #1_{#2})}}{\ensuremath{( #1_{#2})_{ #2 \in #3}}}}

\newcommandx\sequencet[3][2=,3=]
{\ifthenelse{\equal{#3}{}}{\ensuremath{( #1_{#2})}}{\ensuremath{( #1_{#2})_{ #2 \geq #3}}}}

\newcommand{\opnorm}[1]{{\left\vert\kern-0.25ex\left\vert\kern-0.25ex\left\vert #1
    \right\vert\kern-0.25ex\right\vert\kern-0.25ex\right\vert}}

\newcommand\coupling[2]{\Gamma(\mu,\nu)}

\newcommand{\N}{\mathbb{N}}

\newcommand{\R}{\mathbb{R}}

\newcommand{\E}{\mathbb{E}}

\newcommand{\Rd}{\mathbb{R}^{d}}

\newcommand{\argmin}{\operatornamewithlimits{\arg\min}}

\newcommand{\half}{1/2}  
\newcommand{\eqsp}{\,}  

\newcommand{\pr}[1]{\left({#1}\right)}

\newcommand{\br}[1]{\left[{#1}\right]}
\newcommand{\bbr}[1]{\left\{{#1}\right\}}

\newcommand{\norm}[1]{\left\|{#1}\right\|}

\newcommand{\gauss}{\mathrm{N}}

\makeatletter
\newcommand{\ostar}{\mathbin{\mathpalette\make@circled\star}}
\newcommand{\pstar}{\mathbin{\mathpalette\make@circled+}}
\newcommand{\make@circled}[2]{%
  \ooalign{$\m@th#1\smallbigcirc{#1}$\cr\hidewidth$\m@th#1#2$\hidewidth\cr}%
}
\newcommand{\smallbigcirc}[1]{%
  \vcenter{\hbox{\scalebox{0.77778}{$\m@th#1\bigcirc$}}}%
}
\makeatother

\usepackage[accepted]{icml2023}

\usepackage[toc,page,header]{appendix}
\usepackage{minitoc}

\begin{document}
\doparttoc 
\faketableofcontents 

\twocolumn[
\icmltitle{Personalised Federated Learning On Heterogeneous Feature Spaces}
\icmltitlerunning{Personalised Federated Learning On Heterogeneous Feature Spaces\hfill\thepage}

\icmlsetsymbol{equal}{*}

\begin{icmlauthorlist}
	\icmlauthor{Alain Rakotomamonjy}{equal,1}
	\icmlauthor{Maxime Vono}{equal,1}
	\icmlauthor{Hamlet Jesse Medina Ruiz}{1}
	\icmlauthor{Liva Ralaivola}{1}
\end{icmlauthorlist}

\icmlaffiliation{1}{Criteo AI Lab, Paris, France}

\icmlcorrespondingauthor{Maxime Vono}{m.vono@criteo.com}
\icmlcorrespondingauthor{Alain Rakotomamonjy}{a.rakotomamonjy@criteo.com}

\icmlkeywords{Machine Learning, ICML}

\vskip 0.3in
]

\printAffiliationsAndNotice{\icmlEqualContribution} 

\begin{abstract}
    
Most personalised federated learning (FL) approaches assume that raw data of all clients are defined in a common subspace \emph{i.e.} all clients store their data according to the same schema.
For real-world applications, this assumption is restrictive as clients, having their own systems to collect and then store data, may use {\em heterogeneous} data representations. 
We aim at filling this gap.
To this end, we propose a general framework coined \ours that maps client's data onto a common feature space via local embedding functions.
The common feature space is learnt in a federated manner using Wasserstein barycenters while the local embedding functions are trained on each client via distribution alignment. We  integrate this distribution alignement mechanism into
a federated learning approach and provide the algorithmics of \ours. We compare its perfomances against FL benchmarks involving heterogeneous input features spaces. 
In addition, we provide theoretical insights supporting the relevance of our methodology.
\end{abstract}

\section{Introduction}
\label{sec:intro}

Federated learning (FL) is a machine learning paradigm where models are trained from multiple isolated data sets owned by individual agents (coined \emph{clients}), without requiring to move raw data into a central server, nor even share them in any way \citep{KairouzFL}. 
This framework has lately gained a strong traction from both industry and academic research.
Indeed, it avoids the communication costs entailed by data transfer, allows all clients to benefit from participating to the learning cohort, and finally, it fulfills first-order confidentiality guarantees, which can be further enhanced by resorting to so-called \emph{privacy-enhancing technologies} such as differential privacy \citep{10.1561/0400000042} or secure multi-party computation \citep{10.1145/3133956.3133982}.
As core properties, FL ensures data ownership, and structurally incorporates the principle of data exchange minimisation by only transmitting the required updates of the models being learnt.
Depending on the data partitioning and target applications, numerous FL approaches have been proposed, such as horizontal FL \citep{mcmahan17} and vertical FL \citep{FL_concept,VFL2017}. 
The latter paradigm considers that clients hold disjoint subsets of features corresponding to the same users while the former assumes that clients have data samples from different users. 
Recently, horizontal FL works have focused on \emph{personalised} FL to tackle statistical heterogeneity by using local models to fit client-specific data \citep{PFL_review,jiang2019improving,khodak2019adaptive,hanzely2020federated}. 

Existing horizontal personalised FL works assume that the raw data on \emph{all} clients share the same structure and are defined in a common feature space. 
Yet, in practice, data collected by clients may use differing structures.
For instance, clients may not collect exactly the same information, some features may be missing or not stored, or some might have been transformed (\emph{e.g.} via normalisation,  scaling, or linear combinations).  
To address the key issue of implementing FL when the clients' feature spaces are heterogeneous, in the sense that they have different dimensionalites or that the semantics of given vector coordinates are different, we introduce \emph{the first} --- to the best of our knowledge --- personalised FL framework dedicated to this learning situation.

\noindent \textbf{Proposed Approach.} 
The framework and algorithm described in this paper rest on the idea that before performing efficient FL training, a key step is to map the raw data into a common subspace.
This is a prior necessary step before  FL since it allows to define a relevant aggregation scheme on the central server for model parameters (\emph{e.g.} via weighted averaging) as the latter become comparable.
Thus, we map clients' raw data into a common low-dimensional latent space, via local and learnable feature embedding functions. 

In order to ease subsequent learning steps, data related to the same semantic information (\emph{e.g.} label) have to be embedded in the same region of the latent space. 
To ensure this property, we align clients' embedded feature distributions via a latent {\em anchor distribution} that is shared across clients.
The learning of this {\em anchor distribution} is performed in a federated manner \emph{i.e.} by updating it locally on each client before aggregation on the central server. 
More precisely, each client updates her local version of the anchor distribution by aligning it, \emph{i.e.} making it closer, to the embedded feature distribution.
Then, the central server aims at finding the mean element, \emph{i.e.} barycenter, of these local anchor distributions \citep{995827,JMLR:v6:banerjee05b}. 
Once this distribution alignment mechanism (based on local embedding functions
and anchor distribution) is defined, it can be seamlessly integrated into a personalised FL framework; the personalisation part aiming at tackling residual statistical heterogeneity.
In this paper, without loss of generality, we have embedded this alignment framework into a personalised FL approach similar to the one proposed in \citet{pmlr-v139-collins21a}. 
     
\noindent \textbf{Related Ideas.} 
Ideas that we have built on for solving the task of FL from heterogeneous feature spaces have been partially explored in related literature.
From the theoretical standpoint, works on the Gromov-Wasserstein distance or variants seek at comparing distributions from incomparable spaces in a (non-FL) centralised manner \citep{memoli2011gromov,pmlr-v97-bunne19a,alaya2022theoretical}. 
Other methodological works on autoencoders \citep{pmlr-v119-xu20e}, word embeddings \citep{alvarezmelis2018gromov, pmlr-v89-alvarez-melis19a} or FL under high statistical heterogeneity \citep{pmlr-v162-makhija22a,NEURIPS2021_2f2b2656,FedFA} use similar ideas of distribution alignment for calibrating feature extractors and classifiers.
A detailed literature review and comparison with the proposed methodology is postponed to \Cref{sec:methodo}. 

\textbf{Contributions.} In order to help the reader better grasp the differences of our approach with respect to the existing literature, we spell out our contributions:
  
\begin{enumerate}
	 \item We are \emph{the first} to formalise the problem of personalised horizontal FL on heterogeneous clients' feature spaces. 
	 In contrast to existing approaches, the proposed general framework, coined \ours, allows each client to leverage other clients' data even though they do not have the same raw representation.
	 
	\item We introduce a distribution alignment framework and an algorithm that learns the feature embedding functions along with the latent anchor distribution in a local and global federated manner, respectively. We also show how those essential algorithmic pieces are integrated into a personalised FL algorithm, easing adoption by practitioners.
	
	\item We provide algorithmic and theoretical support to the proposed methodology. 
	In particular, we show that for an insightful simpler learning scenario, \ours is able to recover the true latent subspace underlying the FL problem.
	
	\item Experimental analyses on toy data sets and real-world problems illustrate the accuracy of our theory and show that \ours provides better performance than competing FL approaches.
\end{enumerate}

\noindent \textbf{Conventions and Notations.} 
The Euclidean norm on $\mathbb{R}^d$ is $\|\cdot\|$, we use $|\mathrm{S}|$ to denote the cardinality of the set $\mathrm{S}$ and $\nsets = \nset\setminus\{0\}$.
For $n \in \N^*$, we refer to $\{1,\ldots,n\}$ with the notation $[n]$.
We denote by $\mathrm{N}(m,\Sigma)$ the Gaussian distribution with mean vector $m$ and covariance matrix $\Sigma$ and use the notation $X \sim \nu$ to denote that the random variable $X$ has been drawn from the probability distribution $\nu$.
We define the Wasserstein distance of order $2$ for any probability measures $\mu,\nu$ on $\Rd$ with finite $2$-moment by $W_2 (\mu, \nu) = (\inf_{\zeta \in \mathcal{T}(\mu,\nu)} \int_{\mathbb{R}^d \times \mathbb{R}^d}\|\theta-\theta'\|^2\mathrm{d}\zeta(\theta,\theta'))^{\half}$, where $\mathcal{T}(\mu, \nu)$ is the set of transference plans of $\mu$ and $\nu$.


\section{Proposed Methodology}
\label{sec:methodo}

\begin{figure*}
  \begin{center}
    \includegraphics[scale=0.13]{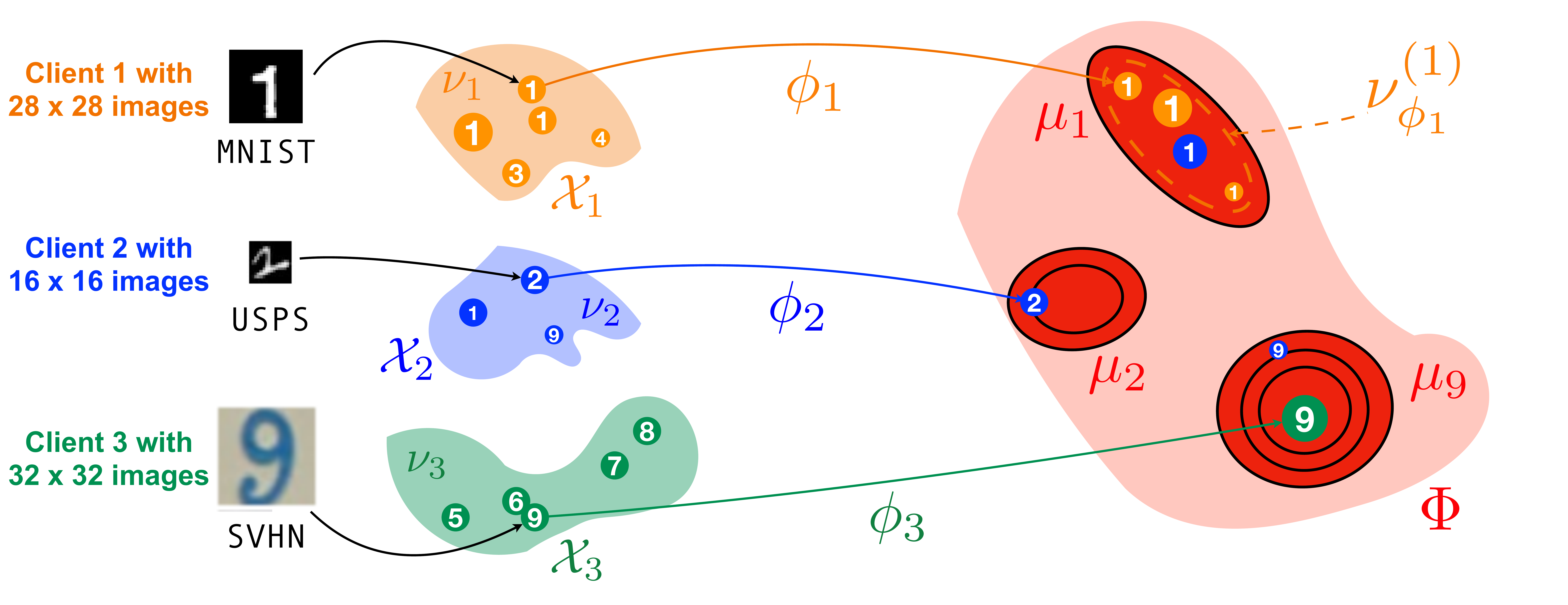}
  \end{center}
  \caption{Illustration of part of the proposed methodology for $b=3$ clients with \emph{heterogeneous} digit images coming from three different data sets namely MNIST \citep{deng2012mnist}, USPS \citep{uspsdataset} and SVHN \citep{SVHN}. The circles with digits inside stand for a group of samples, of a given class, owned by a client and the size of the circles indicates their probability mass.
  In the subspace $\Phi$, $\{\mu_i\}_{i \in [b]}$ (and their level sets) refer to some learnable reference measures to which we seek to align the transformed version $\nu_{\phi_i}$ of $\nu_i$. 
  Personalised FL then occurs in the space $\Phi$ and aims at learning local models $\{\theta_i\}_{i \in [b]}$ for each clients as well as $\{\phi_i,\mu_i\}_{i \in [b]}$.}
  \label{fig:problem_overview}
  \vspace{-0.1cm}
\end{figure*}

\noindent \textbf{Problem Formulation.} We are considering a centralised and horizontal FL framework involving $b \in \N^*$ clients and a central entity \citep{FL_concept,KairouzFL}.
Under this paradigm, the central entity orchestrates the collaborative solving of a common machine learning problem by the $b$ clients; without requiring raw data exchanges.
For the sake of simplicity, we consider the setting where all clients want to solve a multi-class classification task with $C \in \N^*$ classes. 
In Appendix, we also highlight how regression tasks could be encompassed in the proposed framework.
The $b$ clients are assumed to possess local data sets $\{\mathrm{D}_{i}\}_{i \in [b]}$ such that, for any $i \in [b]$, $\mathrm{D}_{i} = \{(x_i^{(j)},y_i^{(j)})\}_{j \in [n_i]}$ where $x_i^{(j)}$ stands for a feature vector, $y_i^{(j)}$ is a label and $n_i = |\mathrm{D}_{i}|$.
A core assumption of FL is that the local data sets $\{\mathrm{D}_{i}\}_{i \in [b]}$ are \emph{statistically heterogeneous} \emph{i.e.} for any $i \in [b]$ and $j \in [n_i]$, $(x_i^{(j)},y_i^{(j)}) \overset{\mathrm{i.i.d.}}{\sim} \nu_i$  where $\nu_i$ is a \emph{local} probability measure defined on an appropriate measurable space.
Existing horizontal FL approaches typically assume that the raw input features of the clients are defined on a common subspace so that their marginal distributions admit the same support.

In contrast, \textbf{we suppose here that these features live in \emph{heterogeneous} spaces}. 
Our main goal is to cope with this new type of heterogeneity in horizontal FL.
More precisely, for any $i \in [b]$ and $j \in [n_i]$, we assume that $x_i^{(j)} \in \mathcal{X}_i \subseteq \R^{k_i}$ such that $\{\mathcal{X}_i\}_{i \in [b]}$ are not part of a common ground metric.
This setting is challenging since standard FL approaches \citep{mcmahan17,FedProx} and even personalised FL ones \citep{pmlr-v139-collins21a,Hanzely2021Personalized} cannot be applied directly.
In addition, we also assume a specific type of \emph{prior probability shift} where, for any $i \in [b]$ and $j \in [n_i]$, $y_i^{(j)} \in \mathcal{Y}_i \subseteq [C]$.
For instance, a client might only see digits $1$ and $7$ from the MNIST data set while another one only has access to USPS digits $1$, $2$ and $9$, see \Cref{fig:problem_overview}.

\noindent \textbf{Methodology.} To address the feature space heterogeneity issue, we propose to map clients' features into a fixed-dimension common subspace $\Phi \subseteq \R^k$ by resorting to \emph{local} embedding functions $\{\phi_i: \mathcal{X}_i \rightarrow \Phi\}_{i \in [b]}$\footnote{Note that we could also have considered push-forward operators acting on the marginals associated to the clients' features, see \citet[Remark 2.5]{OT}.}.
Our proposal for learning those local functions is illustrated in \Cref{fig:problem_overview}. 
In order
to preserve some semantical information (such as the class associated to a feature vector) on the original data distribution, we seek at learning the functions $\{\phi_i\}_{i \in [b]}$ such that they are aligned with (\emph{i.e.} close to) some learnable latent anchor distribution that is shared across all clients. 
This anchor distribution must be seen an universal ``calibrator'' for clients that avoids
similar semantical information from different client being scattered across the subspace $\Phi$, impeding then a proper subsequent federated learning procedure of the classification model. 
%
As depicted in \Cref{fig:problem_overview}, we propose to learn the feature embedding functions by aligning their probability distributions conditioned on the class $c \in [C]$, denoted by $\nu_{\phi_i}^{(c)}$, via $C$ learnable anchor measures $\{\mu_c\}_{c \in [C]}$ \citep{pmlr-v119-xu20e,Tschannen2020On,Kollias2021DistributionMF,Zhou_Chaib-draa_Wang_2021}.
In the literature, several approaches have been considered to align probability distributions ranging from mutual information maximisation \citep{Tschannen2020On}, maximum mean discrepancy \citep{zellinger2017central} to the usage of other probability distances such as Wasserstein or Kullback-Leibler ones \citep{10.5555/3504035.3504532}.

Once data from the heterogeneous spaces are embedded in the same latent subspace $\Phi$, we can deploy a federated learning methodology for training from this novel representation space.
At this step, we need to choose which of standard FL approaches, \emph{e.g.} \texttt{FedAvg} \citep{mcmahan17}, or personalised one are more appropriate.
Since the proposed distribution alignment training procedure via the use of an anchor distribution might not be perfect, some statistical heterogeneity may still appear in the common latent subspace $\Phi$.
Therefore, we aim at solving a \emph{personalised} FL problem where each client has a local model tailored to her specific data distribution in $\Phi$ \citep{PFL_review}. 
By considering an empirical risk minimisation formulation, the resulting data-fitting term we want to minimise writes
\vspace{-0.2cm}
\begin{align}
    f(\theta_{1:b}, \phi_{1:b}) = \sum_{i=1}^b \omega_i f_i(\theta_i,\phi_i)\eqsp, \label{eq:PFL_ERM}
\end{align}
where $\phi_i$ is the aforementioned local embedding function, $\theta_i \in \R^{d_i}$ stands for a local model parameter and $\{\omega_i\}_{i \in [b]}$ are non-negative weights associated to each client such that $\sum_{i=1}^b \omega_i = 1$; and for any $i \in [b]$, 
\vspace{-0.2cm}
\begin{equation}
f_i(\theta_i,\phi_i) = \frac{1}{n_i}\sum_{j=1}^{n_i} \ell\pr{y_i^{(j)}, g_{\theta_i}^{(i)}\br{\phi_i\pr{x_i^{(j)}}}} \eqsp. \label{eq:PFL_local}
\end{equation}
In the local objective function defined in \eqref{eq:PFL_local}, $\ell(\cdot,\cdot)$ stands for a classification loss function between the true label $y_i^{(j)}$ and the predicted one $g_{\theta_i}^{(i)}[\phi_i(x_i^{(j)})]$ where $g_{\theta_i}^{(i)}$ is the local model that admits a personalised architecture parameterised by $\theta_i$ and taking as input an embedded feature vector $\phi_i(x_i^{(j)}) \in \Phi$.


\textbf{Objective Function.}
At this stage, we are able to integrate the FL paradigm and the local embedding function learning into a global objective function we want to optimise, see \eqref{eq:PFL_ERM}. 
Remember that we want to learn the parameters $\{\theta_i\}_{i \in [b]}$ of personalised FL models, in conjuction with some local embedding functions $\{\phi_i\}_{i \in [b]}$ and shared anchor distributions $\{\mu_c\}$.
In particular, the latter have to be aligned with class-conditional distributions $\{\nu_{\phi_i}^{(c)}\}$.
We propose to perform this alignement via a Wasserstein regularisation term leading to consider a regularised version of the empirical risk minimisation problem defined in \eqref{eq:PFL_ERM}, namely
\vspace{-0.3cm}
\begin{align}
    &\theta^\star_{1:b}, \phi^\star_{1:b}, \mu_{1:C}^\star = \argmin_{\theta_{1:b},\phi_{1:b},\mu_{1:C}} \sum_{i=1}^b F_i(\theta_i,\phi_i,\mu_{1:C})\eqsp, \\
    &\text{where for any $i \in [b]$, } \\
    &F_i(\theta_i,\phi_i,\mu_{1:C}) = \omega_i f_i(\theta_i,\phi_i) +\lambda_1\omega_i \sum_{c \in \mathcal{Y}_i} \mathrm{W}_2^2 \pr{\mu_c,\nu_{\phi_i}^{(c)}}\\
    &+\lambda_2 \omega_i \sum_{c \in \mathcal{Y}_i} \frac{1}{J}\sum_{j=1}^J \ell\pr{c, g_{\theta_i}^{(i)}\br{Z_c^{(j)}}}\eqsp, \label{eq:PFL_ERM_reg}
\end{align}
where $\{Z_c^{(j)} ; j \in [J]\}_{c \in [C]}$ stand for samples drawn from $\{\mu_c\}_{c \in [C]}$, and $\lambda_1, \lambda_2 > 0$ are regularisation parameters.
The second term in \eqref{eq:PFL_ERM_reg} aims at aligning the conditional probability measures of the transformed features.
The third one is an optional term aspiring to calibrate the reference measures with the classifier in cases where two or more classes are still ambiguous after mapping onto the common feature space; it has also some benefits to tackle covariate shift in standard FL \citep{NEURIPS2021_2f2b2656}. 

\textbf{Design Choices and Justifications.} In the sequel, we consider the Gaussian anchor measures $\mu_c = \gauss(v_c,\Sigma_c)$ where $v_c \in \R^k$ and $c \in [C]$.
Note that, under this choice, the samples $\{Z_c^{(j)} ; j \in [J]\}_{c \in [C]}$ can be written $Z_c^{(j)} = v_c + L_c \ \xi_c^{(j)}$ where $\xi_c^{(j)} \sim \gauss(0_k,\mathrm{I}_k)$ and $L_c \in \R^{k \times k}$ is such that $\Sigma_c = L_c L_c^\top$ by exploiting the positive semi-definite property of $\Sigma_c$. Invertibility of $L_c$ is ensured by adding a diagonal matrix $\varepsilon \mathrm{I}_k$ with small positive diagonal elements.
One of the key advantages of this Gaussian assumption is that, under mild assumptions, it guarantees the existence of a transport map $T^{(i)}$ such that  $T_{\#}^{(i)}(\nu_i) = \mu$, owing to Brenier's theorem \citep{santambrogio2015optimal} as a mixture of Gaussians admits a density with respect to the Lebesgue measure. 
Hence, in our case, learning the local embedding functions boils down to approximating this transport map by $\phi_i$. 
In addition, sampling from a Gaussian probability distribution can be performed efficiently \citep{Vono2020_gaussian,Parker2012,Gilavert2015}, even in high dimension.
We also consider approximating the conditional probability measures $\{\nu_{\phi_i}^{(c)} ; c \in \mathcal{Y}_i\}_{i \in [b]}$ by using Gaussian measures $\{\hat{\nu}_{\phi_i}^{(c)} = \gauss(\hat{m}_i^{(c)},\hat{\Sigma}_i^{(c)}) ; c \in \mathcal{Y}_i\}_{i \in [b]}$ such that for any $i \in [b]$ and $c \in [C]$, $\hat{m}_i^{(c)}$ and $\hat{\Sigma}_i^{(c)}$ stand for empirical mean vector and covariance matrix.
The relevance of this approximation is detailed in \Cref{subsec:Wass}.

These two Gaussian choices (for the anchor distribution and the class-conditional distributions) allow us to have a closed-form expression for the Wasserstein distance of order 2 which appears in \eqref{eq:PFL_ERM_reg}, see \emph{e.g.} \citet{https://doi.org/10.1002/mana.19901470121,DOWSON1982450}.
More precisely, we have for any $i \in [b]$ and $c \in [C]$,
\begin{align}
    \mathrm{W}_2^2 \pr{\mu_c,\nu_{\phi_i}^{(c)}} &= \norm{v_c - m_i^{(c)}}^2 + \mathfrak{B}^2\pr{\Sigma_c   , \Sigma_i^{(c)}}, \label{eq:W2_Gauss}
\end{align}
where $\mathfrak{B}(\cdot,\cdot)$ denotes the Bures distance between two positive definite matrices \citep{BHATIA2019165}. 
In addition to yield the closed-form expression \eqref{eq:W2_Gauss}, the choice of the Wasserstein distance is motivated by two other important properties.
First, it is always finite no matter how degenerate the Gaussian distributions are, contrary to other divergences such as the Kullback-Leibler one \citep{DBLP:journals/corr/VilnisM14}. 
Being able to output a meaningful distance value when supports of distribution do not overlap is a key benefit of the Wasserstein distance, since when initialising $\phi_i$, we do not have any guarantee on such
overlapping (see illustration given in \Cref{fig:tsne-toybig}).
Second, its minimisation can be handled using efficient algorithms proposed in the optimal transport literature \citep{NEURIPS2018_b613e70f}. 

\begin{table*}[t]
\tiny
\caption{Related works. PFL refers to horizontal personalised FL, VFL to vertical FL and FTL to federated transfer learning.}
\label{table:overview}
\vskip 0.15in
\begin{center}
{\small
\begin{sc}
\begin{tabular}{lcccccc}
\toprule
method & type & $\neq$ feature spaces & multi-party & no shared ID & no shared feature \\
\midrule
\citep{zhang2021parameterized} & PFL & \redcross & \greencheck & \greencheck & \redcross \\
\citep{diao2021heterofl} & PFL & \redcross & \greencheck & \greencheck & \redcross \\
\citep{pmlr-v139-collins21a} & PFL & \redcross & \greencheck & \greencheck & \redcross \\
\citep{pmlr-v139-shamsian21a} & PFL & \redcross & \greencheck & \greencheck & \redcross \\
\citep{hong2022efficient} & PFL & \redcross & \greencheck & \greencheck & \redcross \\
\citep{pmlr-v162-makhija22a} & PFL & \redcross & \greencheck & \greencheck & \greencheck \\
\ours (this paper) & PFL & \greencheck  & \greencheck & \greencheck & \greencheck  \\
\midrule
\citep{VFL2017} & VFL & \greencheck & \redcross & \redcross & \greencheck \\
\citep{FL_concept} & VFL & \greencheck & \redcross & \redcross & \greencheck \\
\midrule
\citep{HFTL_Gao} & FTL & \greencheck & \greencheck & \greencheck & \redcross \\
\citep{secureFTL} & FTL & \redcross & \redcross & \greencheck & \redcross \\
\citep{LiuFTL} & FTL & \greencheck & \redcross & \redcross & \greencheck \\
\citep{CFTL2022} & FTL & \greencheck & \greencheck & \redcross & \redcross \\
\bottomrule
\end{tabular}
\end{sc}
}
\end{center}
\vskip -0.1in
\end{table*}

\noindent \textbf{Related Work.} As pointed out in \Cref{sec:intro}, several existing works can be related to the proposed methodology.
Loosely speaking, we can divide these related approaches into three categories namely (i) heterogeneous-architecture personalised FL, (ii) vertical FL and (iii) federated transfer learning.

Compared to traditional horizontal personalised FL (PFL) approaches, so-called \emph{heterogeneous-architecture} ones are mostly motivated by local heterogeneity regarding resource capabilities of clients \emph{e.g.} computation and storage \citep{zhang2021parameterized,diao2021heterofl,pmlr-v139-collins21a,pmlr-v139-shamsian21a,hong2022efficient,pmlr-v162-makhija22a}.
Nevertheless, they never consider features defined on heterogeneous subspaces, which is our main motivation.
In vertical federated learning (VFL), clients hold disjoint subsets of features. 
However, a restrictive assumption is that a large number of users are common across the clients \citep{FL_concept,VFL2017,angelou2020asymmetric,romanini2021pyvertical}. 
In addition, up to our knowledge, no vertical personalised FL approach has been proposed so far, which is restrictive if clients have different business objectives and/or tasks.
Finally, some works have focused on adapting standard tranfer learning approaches with heterogeneous feature domains under the FL paradigm.
These \emph{federated transfer learning} (FTL) approaches \citep{HFTL_Gao,CFTL2022,LiuFTL,secureFTL} stand for FL variants of heterogeneous-feature transfer learning where there are $b$ \emph{source} clients and 1 target client with a target domain.
However, these methods do not consider the same setting as ours and assume that clients share a common subset of features.
We compare the most relevant approaches among the previous ones in \Cref{table:overview}.

\section{Algorithm}
\label{sec:algo}

As detailed in Equation \eqref{eq:PFL_ERM_reg}, we perform personalisation under the FL paradigm by considering local model architectures $\{g_{\theta_i}^{(i)}\}_{i \in [b]}$ and local weights $\theta_{1:b}$.
As an example, we could resort to federated averaging with fine-tuning (\emph{e.g.} \texttt{FedAvg-FT}, see \citet{collins_neurips22}), model interpolation (\emph{e.g.} \texttt{L2GD}, see \citet{hanzely2020federated,hanzely2020lower}) or partially local models (\emph{e.g.} \texttt{FedRep}, see \citet{oh2022fedbabu,NEURIPS2021_5d44a2b0,pmlr-v139-collins21a}). 
\Cref{table:perso_algo} details how these methods can be embedded into the proposed methodology. 
\vspace{-0.3cm}
\begin{table}[H]
    \caption{Current personalised FL techniques that can be embedded in the proposed framework. The parameters $\alpha,\beta_i$ stand for model weights while $\omega \in [0,1]$.}
    \label{table:perso_algo}
    \vskip 0.15in
    \begin{center}
        \begin{tabular}{ccc}
            \toprule
            Algorithm & Local model & Local weights \\
            \midrule
            {\small\texttt{FedAvg-FT}} & $g^{(i)}_{\theta_i} = g_{\theta_i}$ & $\theta_i$\\
            {\small\texttt{L2GD}} & $g^{(i)}_{\theta_i} = g_{\theta_i}$ & $\theta_i = \omega\alpha + (1-\omega)\beta_i$\\
            {\small\texttt{FedRep}} & $g^{(i)}_{\theta_i} = g^{(i)}_{\beta_i} \circ g_{\alpha}$ & $\theta_i = [\alpha,\beta_i] $\\
            \bottomrule
        \end{tabular}
    \end{center}
    \vspace{-0.5cm}
\end{table} 

In \Cref{algo:FLIC-FEDREP}, we detail the pseudo-code associated to a specific instance of the proposed methodology when \texttt{FedRep} is resorted to learn model parameters $\{\theta_i = [\alpha,\beta_i]\}_{i \in [b]}$ under the FL paradigm. 
In this setting, $\alpha$ stand for the shared weights associated to the first layers of a neural network architecture and $\beta_i$ for local ones aiming at performing personalised classification. 
Besides these two learnable parameters, the algorithm also learns
the local embedding functions $\phi_{1:b}$ and the anchor distribution $\mu_{1:C}$. 
In practice, at a given epoch $t$ of the algorithm, a subset $\mathsf{A}_{t+1} \subseteq [b]$ of clients are selected to participate to the training process. 
Those clients receive the current latent anchor distribution $\mu_{1:C}^{(t)}$ and the current shared representation $\alpha^{(t)}$. 
Then, each client locally updates $\phi_i$, $\beta_i$ and her local versions of $\alpha^{(t)}$ and $\mu_{1:C}^{(t)}$. 
Afterwards, clients send back to the server an updated version of $\alpha^{(t)}$ and $\mu_{1:C}^{(t)}$. 
Updated global parameters $\alpha^{(t+1)}$ and $\mu_{1:C}^{(t+1)}$ are then obtained by weighted averaging of client updates on appropriate manifolds. 
The use of the Wasserstein loss in \eqref{eq:PFL_ERM_reg} naturally leads to perform averaging of the local anchor distributions via a Wasserstein barycenter; algorithmic details are provided in the next paragraph.
In \Cref{algo:FLIC-FEDREP}, we use for the sake of simplicity the notation \texttt{DescStep}($F_i^{(t,m)}, \cdot)$ to denote a (stochastic) gradient descent step on the function $F_i^{(t,m)} = F_i(\beta_i^{(t,m)},\phi_i^{(t,m)},\alpha^{(t)},\mu_{1:C}^{(t)})$ with respect to a subset of parameters in $(\theta_i,\phi_i,\mu_{1:C})$. This subset is specified in the second argument of \texttt{DescStep}.
An explicit version of \Cref{algo:FLIC-FEDREP} is provided in Appendix, see \Cref{algo:FLIC-FEDREP_supp}.

\begin{algorithm}[t]
	\caption{\textcolor{magenta}{\texttt{FLIC}}}
	\label{algo:FLIC-FEDREP}
	\begin{algorithmic}[1]
		\REQUIRE{ initialisation $\alpha^{(0)}$, $\mu_{1:C}^{(0)}$, $\phi_{1:b}^{(0,0)}$, $\beta_{1:b}^{(0,0)}$.}
		\FOR{$t=0$ {\bfseries to} $T-1$}
			\STATE Sample a set of $\mathsf{A}_{t+1}$ of active clients.
			\FOR{$i \in \mathsf{A}_{t+1}$} 
			\STATE The central server sends $\alpha^{(t)}$ and $\mu_{1:C}^{(t)}$ to $\mathsf{A}_{t+1}$.
			\STATE {\textcolor{darkgreen}{\textit{// Update local parameters}}}
			\FOR{$m=0$ {\bfseries to} $M-1$}
				\STATE  $\phi_i^{(t,m+1)} \leftarrow  \texttt{DescStep}\pr{F_i^{(t,m)},\phi_i^{(t,m)}}$.
				\STATE $\beta_i^{(t,m+1)} \leftarrow \texttt{DescStep}\pr{F_i^{(t,m)},\beta_i^{(t,m)}}$.
			\ENDFOR
			\STATE $\phi_i^{(t+1,0)} = \phi_i^{(t,M)}$.
			\STATE $\beta_i^{(t+1,0)} = \beta_i^{(t,M)}$.
			
			\STATE {\textcolor{darkgreen}{\textit{// Update global parameters}}}
			\STATE $\alpha_i^{(t+1)} \leftarrow \texttt{DescStep}\pr{F_i^{(t,M)},\alpha^{(t)}}$.

			\STATE$\mu_{i,1:C}^{(t+1)} \leftarrow  \texttt{DescStep}\pr{F_i^{(t,M)},\mu_{1:C}^{(t)}}$.

			\STATE {\textcolor{darkgreen}{\textit{//~Communication with the server}}}
			\STATE Send $\alpha_i^{(t+1)}$ and $\mu_{i,1:C}^{(t+1)}$ to central server.
		\ENDFOR
		\STATE {\textcolor{darkgreen}{\textit{//~Averaging global parameters}}}
		\STATE $\alpha^{(t+1)} = \frac{b}{|\mathsf{A}_{t+1}|}\sum_{i \in \mathsf{A}_{t+1}} w_i \alpha_i^{(t+1)}$
		\STATE $\mu_{1:C}^{(t+1)} \leftarrow \texttt{WassersteinBarycenter}(\{\mu_{i,1:C}^{(t+1)}\} )$ 
		\ENDFOR
		\ENSURE parameters $\alpha^{(T)}$, $\mu_{1:C}^{(T)}$, $\phi_{1:b}^{(T,0)}$, $\beta_{1:b}^{(T,0)}$. 
\end{algorithmic}
\end{algorithm}

Note that we take into account key inherent challenges to federated learning namely \emph{partial participation} and \emph{communication bottleneck}.
Indeed, we cope with the client/server upload communication issue by allowing each client to perform multiple steps (here $M \in \N^*$) so that communication is only required every $M$ local steps. This allows us to consider updating global parameters, locally, via only one stochastic gradient descent step and hence avoiding the client drift phenomenon \citep{karimireddy2020scaffold}.

\noindent \textbf{Averaging Anchor Distributions.} In this paragraph, we provide algorithmic details regarding steps 14 and 20 in \Cref{algo:FLIC-FEDREP}. 
For any $c \in [C]$, the anchor distribution $\mu_c$ involves two learnable parameters namely the mean vector $v_c$ and the covariance matrix $\Sigma_c$. 
Regarding the former, step 14 stands for a (stochastic) gradient descent step aiming to obtain a local version of $v_c$ denoted by $v_{i,c}^{(t+1)}$ and step 20 boils down to compute $v_c^{(t+1)} = (b/|\mathsf{A}_{t+1}|) \sum_{i \in \mathsf{A}_{t+1}} \omega_i v_{i,c}^{(t+1)}$.
To enforce the positive semi-definite constraint of the covariance matrix, we rewrite it as $\Sigma_c = L_c  L_c^\top$ where $L_c \in \R^{k \times k}$ and optimise in step 14 with respect to the factor $L_c$ instead of $\Sigma_c$. 
We can handle the gradient computation of the Bures distance in step 14 using  the work of \citet{NEURIPS2018_b613e70f}; and obtain a local factor $L_{i,c}^{(t+1)}$ at iteration $t$.
In step 20, we compute $L_c^{(t+1)} = (b/|\mathsf{A}_{t+1}|) \sum_{i \in \mathsf{A}_{t+1}} \omega_i L_{i,c}^{(t+1)}$ and set $\Sigma_c^{(t+1)} = L_c^{(t+1)} [L_c^{(t+1)}]^\top$.
When $\lambda_2 = 0$ in \eqref{eq:PFL_ERM_reg}, these mean vector and covariance matrix updates exactly boil down to perform one stochastic (because of partial participation) gradient descent step to solve the Wasserstein barycenter problem $\argmin_{\mu_c} \sum_{i=1}^b \omega_i \mathrm{W}_2^2 (\mu_c,\nu_{\phi_i}^{(c)})$.  



%
%


\begin{table*}
	\caption{Performance over 3  runs of our \ours model and the competitors on some real-data problems (\emph{Digits} and \emph{TextCaps} data set).\\}%
	\centering%
		\begin{tabular}{lcccc}%
			\toprule
			Data sets (setting) & Local             &  \texttt{FedHeNN} & \ours-Class & \ours-HL   \\
			\midrule
			 Digits  ($b=100$, 3 Classes/client)           & 97.49 & 97.45 & \textbf{97.83} & 97.70     \\
			 Digits  ($b=100$, 5 Classes/client)           & 96.16 & 96.15 & \textbf{96.46} & 96.31     \\
			 Digits  ($b=200$, 3 Classes/client)           & 93.33 & 93.40 & 94.50 & \textbf{94.51}    \\
			 Digits  ($b=200$, 5 Classes/client)           & 86.50 & 87.22 & \textbf{90.66} & 90.63   \\
			\midrule		
			 TextCaps  ($b=100$, 2 Classes/client)           & 84.19 & 83.99 & 89.14 & \textbf{89.68}      \\
			TextCaps  ($b=100$, 3 Classes/client)          & 76.04 & 75.39 & 81.27 & \textbf{81.50}  \\
			TextCaps*  ($b=200$, 2 Classes/client)          & 83.78 & 83.89 & 87.73 & \textbf{87.74}         \\
			TextCaps*  ($b=200$, 3 Classes/client)            & 74.95 & 74.77 & \textbf{79.08} & 78.49   \\
 			\bottomrule
	\end{tabular}
	\label{table:performance}%
\end{table*}

\section{Non-Asymptotic Convergence Guarantees in a Simplified Setting}
\label{sec:theory}

\begin{figure}
  \begin{center}
    \includegraphics[scale=0.35]{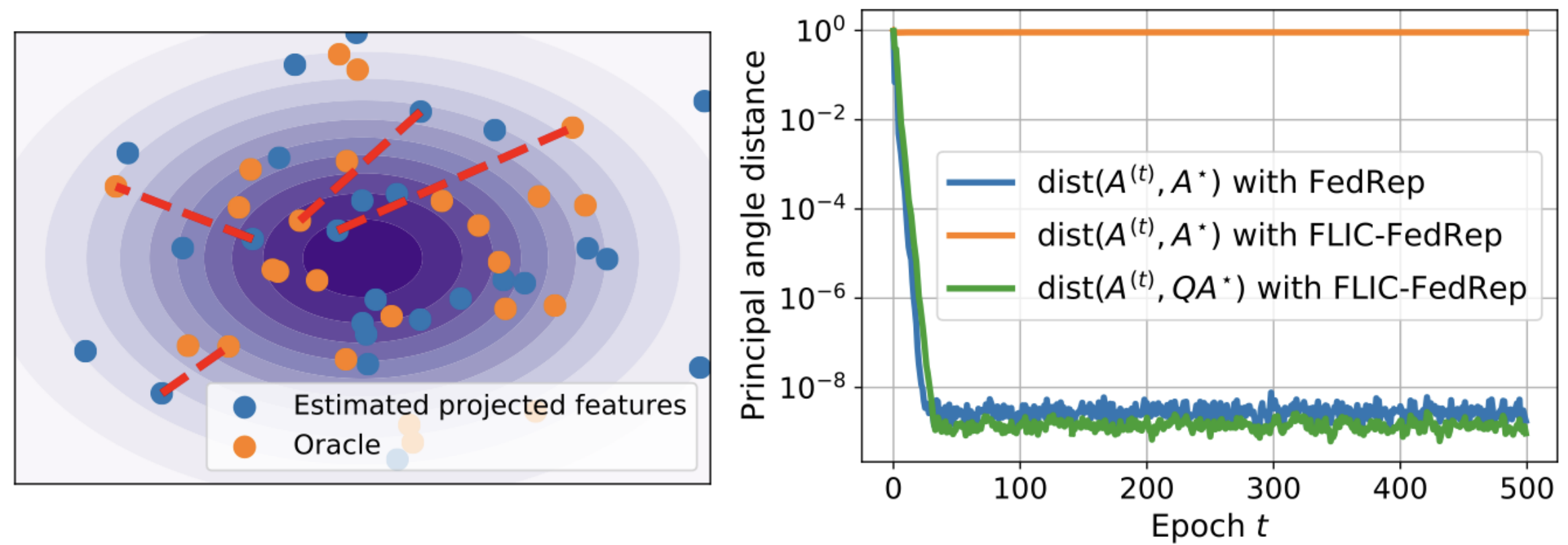}
  \end{center}
  \caption{Red dashed line indicates that the two embedded features $\phi_i^\star(x_i^{(j)})$ and $\hat{\phi}_i(x_i^{(j)})$ come from the same initial raw feature $x_i^{(j)}$. On test data, mean prediction errors for both \texttt{FedRep} operating on $\phi_i^\star(x_i^{(j)})$ and \Cref{algo:FLIC-FedRep-linreg} (referred to as \texttt{FLIC-FedRep}) are similar ($\approx 4.98 \times 10^{-5}$).}
  \label{fig:small_data}
    \label{fig:conv_Q}
\end{figure}

Deriving non-asymptotic convergence bounds for \Cref{algo:FLIC-FEDREP} in the general case is challenging  since the considered $C$-class classification problem leads to jointly solving personalised FL and federated Wasserstein barycenter problems. Regarding the latter, obtaining non-asymptotic convergence results is still an active research area in the centralised learning framework \citep{altschuler2021averaging}.
As such, we propose to analyse a simpler regression framework where the anchor distribution is known beforehand and not learnt under the FL paradigm.

More precisely, we assume that $x_i^{(j)} \sim \mathrm{N}(m_i,\Sigma_i)$ with $m_i \in \R^{k_i}$ and $\Sigma_i \in \R^{k_i \times k_i}$ for $i \in [b], j \in [n_i]$.
In addition, we consider that the continuous scalar labels are generated via the oracle model $y_i^{(j)} = (A^\star\beta^\star_i)^\top \phi_i^\star(x_i^{(j)})$ where $A^\star \in \R^{k \times d}$, $\beta_i^\star \in \R^d$ and $\phi_i^\star(\cdot)$ are ground-truth parameters and feature transformation function, respectively. 
We make the following assumptions on the ground truth.

\begin{assumption}
    \label{ass}
    \begin{enumerate}[wide, labelwidth=!, labelindent=0pt,label=(\roman*),noitemsep,nolistsep]

        \item For any $i \in [b]$, $j \in [n_i]$, embedded features $\phi_i^\star(x_i^{(j)})$ are distributed according to $\mathrm{N}(0_k,\mathrm{I}_k)$.\label{ass:1}
        \item Ground-truth model parameters satisfy $\|\beta_i^\star\|_2 = \sqrt{d}$ for $i \in [b]$ and $A^\star$ has orthonormal columns.
        \item For any $t \in \{0,\ldots,T-1\}, |\mathsf{A}_{t+1}| = b'$ with $1 \leq b' \leq b$, and if we select $b'$ clients, their ground-truth head parameters $\{\beta_i^\star\}_{i \in \mathsf{A}_{t+1}}$ span $\R^d$.
        \item In \eqref{eq:PFL_local}, $\ell(\cdot,\cdot)$ is the $\ell_2$ norm, $\omega_i=1/b$, $\theta_i = [A,\beta_i]$ and $g_{\theta_i}^{(i)}(x) = (A\beta_i)^\top x$ for $x \in \R^k$.
    \end{enumerate}
\end{assumption}

Under \Cref{ass}-\ref{ass:1}, \citet[Theorem 4.1]{delon_desolneux_salmona_2022} show that $\phi_i^\star$ can be expressed as a non-unique  affine map with closed-form expression.
To recover the true latent distribution $\mu = \mathrm{N}(0_k,\mathrm{I}_k)$, we propose to estimate $\hat{\phi}_i$ by leveraging this closed-form mapping between $\mathrm{N}(m_i,\Sigma_i)$ and $\mu$.
Because of the non-unicity of $\phi_i^\star$, we show in \Cref{theorem2} that we can only recover it up to a matrix multiplication. 
Interestingly, \Cref{theorem2} shows that the global representation $A^{(T)}$ learnt via \texttt{FedRep} (see \Cref{algo:FLIC-FedRep-linreg} in Appendix) is able to correct this feature mapping indetermination. 
Associated convergence behavior is illustrated in \Cref{fig:conv_Q} on a toy example whose details are postponed to \Cref{sec:proof}. 

\begin{theorem}
    \label{theorem2}
    Assume \Cref{ass}.
    Then, for any $x_i \in \R^{k_i}$, we have $\hat{\phi}_i(x_i) = Q \phi_i^\star(x_i)$ where $Q \in \R^{k \times k}$ is of the form $\mathrm{diag}_k(\pm 1)$.
    Under additional technical assumptions detailed in \Cref{sec:proof}, we have for any $t \in \{0,\ldots,T-1\}$ and with high probability,
    $$
    \mathrm{dist}(A^{(t+1)},QA^\star) \leq (1 - \kappa)^{(t+1)/2} \mathrm{dist}(A^{(0)},QA^\star)\eqsp,
    $$
    where $\kappa \in (0,1)$ is detailed explicitly in \Cref{theorem:1} and $\mathrm{dist}$ denotes the principal angle distance. 
    \label{theorem:cv_FedRep}
\end{theorem}

\section{Numerical Experiments}


In this section, we aim at illustrating how our algorithm \ours, when using \texttt{FedRep} as FL approach, works in practice and showcasing its numerical performances. 
We consider several toy problems with different characteristics of heterogeneity; as well as experiments on real data namely a digit classification problem from images of different sizes and an object classification problem from either images or text captioning on clients.

\noindent \textbf{Baselines.} Since the problem we are addressing is novel, there exists no FL competitor that can serve as a baseline beyond local learning. 
However, we propose to modify the methodology proposed in \citet{pmlr-v162-makhija22a} to make it applicable to clients with heterogeneous feature spaces.
This latter approach can handle local representation models with different architectures and the key idea, coined Representation Alignement Dataset (RAD), is to calibrate those models by matching the latent representation of some fixed data inputs shared by the server to all clients. 
In our case, we can not use the same RAD accross all clients due to the different dimensionality of the local models. A simple alternative that we consider in our experiments is to build a RAD given the largest dimension space among all clients and then prune it to obtain a lower-dimensional RAD suitable to each client. We refer to the corresponding algorithm as \texttt{FedHeNN}.

We are going to consider the same architecture networks for all baselines. 
As \citet{pmlr-v162-makhija22a} considers all but the last layer of the network as the representation learning module, for a fair comparison, we also assume the same for our approach. Hence, in our case, the last layer is the classifier layer and the alignment with the latent reference distribution applies on the penultimate layer. This model is referred to as \ours-Class, in which all weights are thus local ($\alpha$ is empty and $\beta_i$ refers
to the last layer).
In addition, we also have
a model, coined \ours-HL, which has an additional trainable global hidden layer, which $\alpha$ being the parameter of linear layer and $\beta_i$ the parameter of the classification layer.

\noindent \textbf{Data Sets.} We consider three different classification problems to assess the performances of our approach. 
First, we are considering a toy classification problem with $C=20$ classes and where each class-conditional distribution is a Gaussian with random mean.
Covariance matrices of all classes are the same and considered fixed.
Using this toy data set, we are considering two sub-experiments.
For the first one, named \emph{noisy features} (and labelled \emph{toy NF} in figures), we consider a $5$-dimensional problem ($k=5$) and for each client add some random spurious features which dimensionality goes up to $10$. Hence, in this case
$k_i \in [5,15]$.  
For the second sub-experiment, denoted as \emph{linear mapping} (and labelled \emph{toy LM} in figures), we  apply a Gaussian random linear mapping to the original data which are of dimension $30$. The output dimension of the mapping is uniformly drawn from $5$ to $30$ leading to $k_i \in [5,30]$.  More details are provided in Appendix.

The second problem we consider is a digit classification problem with the original MNIST and USPS data sets which are respectively of dimension $28 \times 28$ and $16 \times 16$ and we assume that a client hosts either a subset of MNIST or USPS dataset.
Finally, the last classification problem is associated to a subset of the \emph{TextCaps} data set \citep{sidorov2020textcaps}, which is an image captioning data set, that we convert into a $4$-class classification problem, 
with about $12,000$ and $3,000$ examples for respectively training and testing, either based on the caption text or the image. 
Some examples of image/caption pairs as well as as more details on how the dataset has been obtained are shown in the \Cref{fig:textcaps}.
The caption has been embedded into a $768$-dimensional vector using a pre-trained Bert embedding and the image into a $512$-dimensional ones using a pre-trained ResNet model.
We further generated some heterogeneity by randomly pruning $10\%$ of these features on each client.
Again, we assume that a client hosts either some image embeddings or text embeddings.
For all simulations, we assume prior probability shift \emph{e.g} each client will have access to data of only specific classes.
   
\noindent \textbf{Experimental Setting.} For the experimental analysis, we use the codebase of \citet{pmlr-v139-collins21a} with some modifications to meet our setting.
For all experiments, we consider $T=50$ communication rounds for all algorithms; and at each round, a client participation rate of $r=0.1$. The number of local epochs for training
has been set to $M=10$.  As optimisers, we have used 
an Adam optimiser with a  learning rate of $0.001$ for all problems and approaches.  Further details are given in \Cref{append:models}.
For each component of the latent anchor distribution, we consider a Gaussian with learnable mean vectors and fixed Identity covariance matrix. As such, the Wasserstein barycenter computation boils
down to simply average the mean of client updates and for
computing the third term in   \Cref{eq:PFL_ERM_reg}, we just sample from
the Gaussian distribution.
 Accuracies are computed as
the average accuracy over all clients after the last epoch in which all local models are trained. 


\noindent \textbf{Results on Toy Data Sets.} \Cref{fig:toyperf} depicts the performance, averaged over $5$ runs, of the different algorithms with respect to the number of clients and when only $3$ classes are present in each client. 
For both data sets, we can note that for the \textit{noisy feature} setting, \ours improves on FedHeNN of about $3\%$ of accuracy across the setting and performs better than local learning. 
For the \textit{linear mapping} setting, \ours achieves better than other approaches with a  gain of performance of about $4\%$ while the gap tends to decrease as the number of clients increases. Interestingly, \ours-HL performs slightly better
than \ours-Class showing the benefit of using a shared representation layer $\alpha$.

\begin{figure}[t]
	\begin{center}
		\includegraphics[scale=0.228]{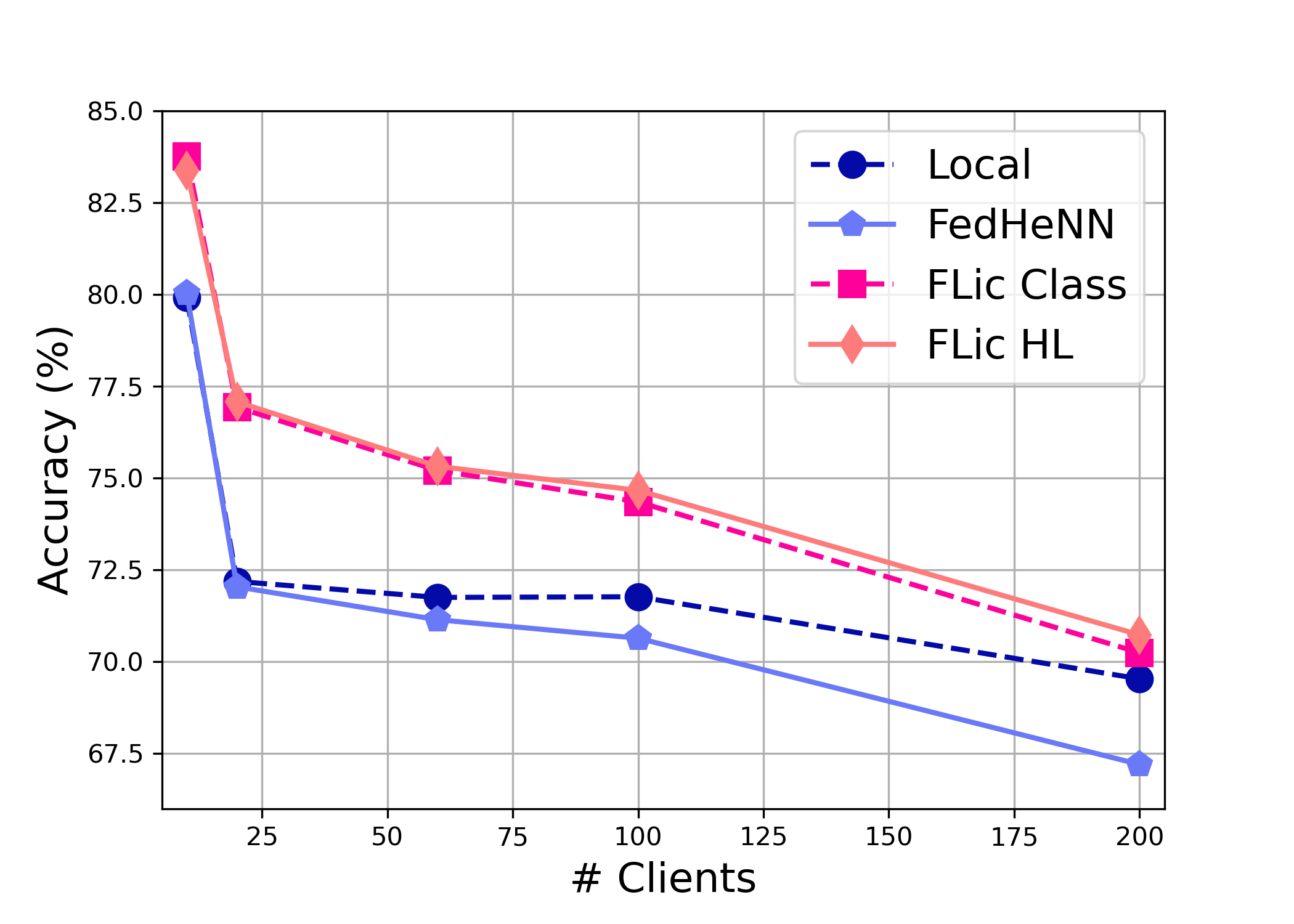}
		\includegraphics[scale=0.228]{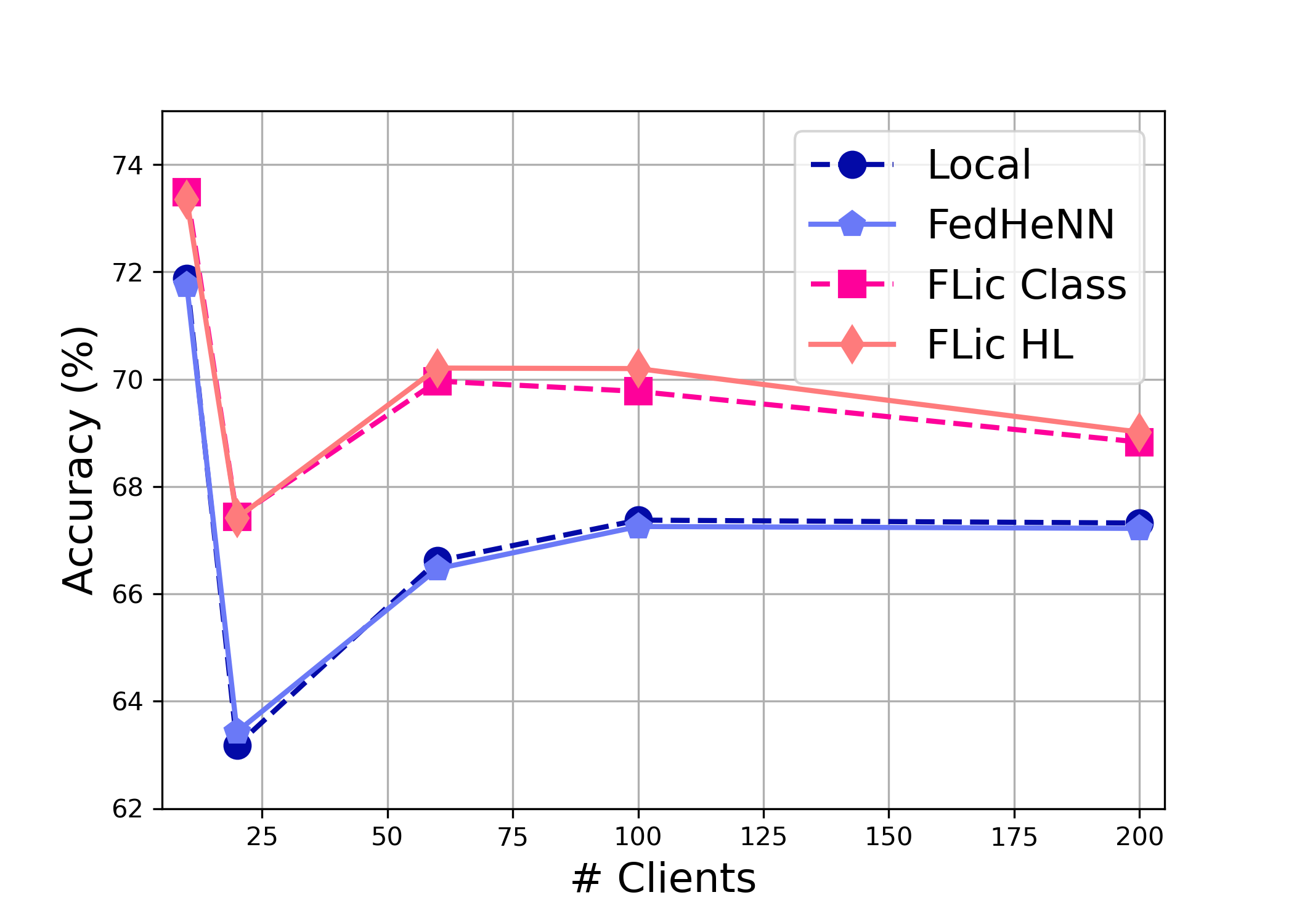} \\
	\end{center}
	\caption{Performance of \ours and competitors on the toy data sets with respect to the number of clients. (left) Gaussian classes in dimension $k=5$ with added noisy feature. (right) Gaussian classes in dimension $k=30$, transformed by a random map. Only $3$ classes are present on each client among the $20$ possible ones.}
	\label{fig:toyperf}
\end{figure}

\begin{figure}[t]
	\begin{center}
		\includegraphics[scale=0.25]{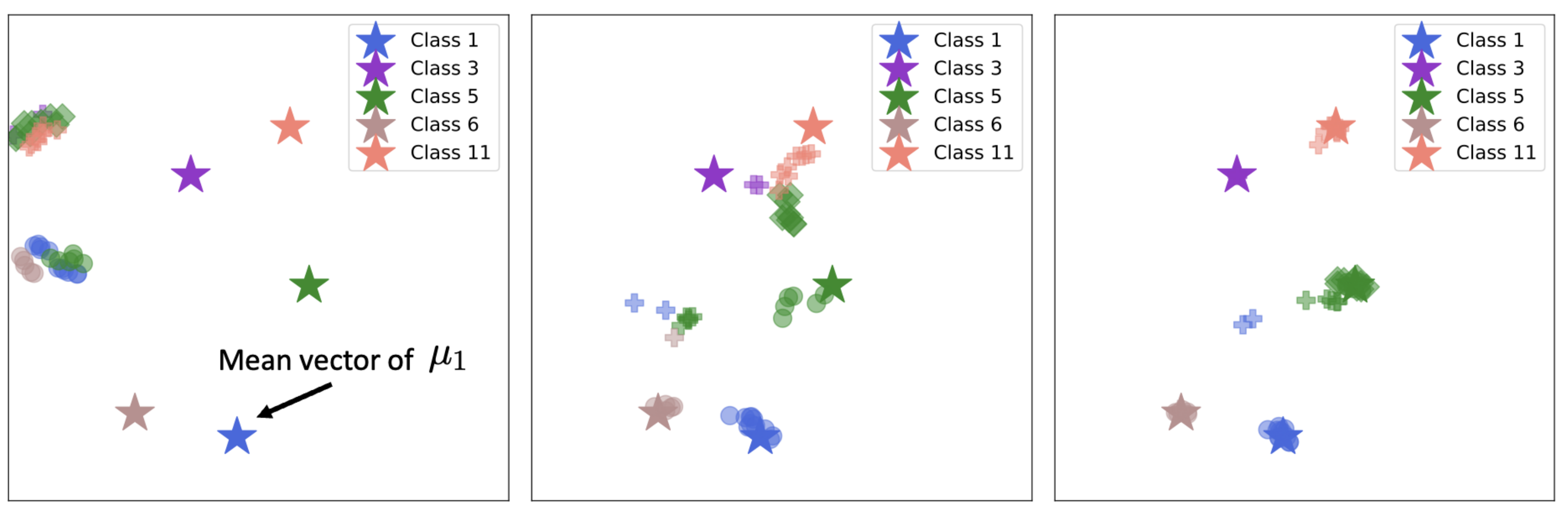}
	\end{center}
	\caption{\label{fig:tsne-toy}. 2D \emph{t-sne} projection of $5$ classes partially shared by $3$ clients for the \textbf{toy LM} dataset after learning the local
			transformation functions for (left) 10 epochs, (middle) 50 epochs, (right) 100 epochs. The three different markers represent the different clients. the 
			$\star$ marker represents the class-conditional mean of the reference distribution. We note that training set converges towards those means. }
	\label{fig:centroid}
	\vspace{-0.1cm}
\end{figure}

\noindent \textbf{Results on Digits and TextCaps Data Sets.}
Performance, averaged over $3$ runs, of all algorithms on the real-word problems are reported in \Cref{table:performance}. For the \emph{Digits} data set problem, we first remark that in  all situations, FL algorithms performs a bit better than local learning.
In addition, both variants of \ours achieve better accuracy than competitors. Difference in performance in favor our \ours reaches $3\%$ for the most
difficult problem.
For the $\emph{TextCaps}$ data set, gains in performance of \ours-HL reach about $4\%$ across settings.  While \texttt{FedHeNN} and \ours algorithms follow the same underlying principle (alignment of representation in a latent space),  we believe  that our framework benefits from the use of the latent anchor distributions, avoiding the need of sampling from the original space. Instead, \texttt{FedHeNN}  may fail as the sampling strategy of their RAD approach  suffers from the curse of dimensionality and does not properly lead to a successful  feature alignment.

\noindent \textbf{Additional Experiments in Appendix.} Due to the limited number of pages, additional experiments are postponed to the Appendix. In particular, we investigate the impact of pre-training the local embedding functions for a fixed reference distribution as in \Cref{sec:theory}, before running the proposed algorithm detailed in \Cref{algo:FLIC-FEDREP}.
The main message is that pre-training helps in enhancing performance but may lead to overfitting
if too many epochs are considered.
We also analyse the impact of client participation rate at each communication round reaching the conclusion that our model is robust to participation rate.


\section{CONCLUSION}

We have introduced a new framework for personalised FL when clients have heterogeneous feature spaces. 
We proposed a novel FL algorithm involving two key components: (i) a local feature embedding function; and (ii) a latent anchor distribution which allows to match similar semantical information from each client. 
Experiments on relevant data sets have shown that  \ours  achieves better performances than competing approaches.
Finally, we provided theoretical support to the proposed methodology, notably via a non-asymptotic convergence result.

\clearpage
\newpage
\bibliography{biblio}
\bibliographystyle{plainnat}

\clearpage
\newpage

\appendix

\onecolumn

\addcontentsline{toc}{section}{Appendix} 
\part{Appendix} 
\parttoc 

\newtheorem{unlemma}{Lemma S}
\newtheorem{unproposition}{Proposition S}
\newtheorem{uncorollary}{Corollary S}
\newtheorem{untheorem}{Theorem S}

\setcounter{equation}{0}
\setcounter{figure}{0}
\setcounter{table}{0}
\setcounter{assumption}{0}
\makeatletter
\renewcommand{\theequation}{S\arabic{equation}}
\renewcommand{\thefigure}{S\arabic{figure}}
\renewcommand{\thetheorem}{S\arabic{theorem}}
\renewcommand{\thelemma}{S\arabic{lemma}}
\renewcommand{\thetable}{S\arabic{table}}
\renewcommand{\thesection}{S\arabic{section}}
\renewcommand{\thealgorithm}{S\arabic{algorithm}}

{\begin{center}\Large\textbf{SUPPLEMENTARY MATERIAL}\end{center}}\vspace{1cm}

\paragraph{Notations and conventions.}

We denote by $\mathcal{B}\parentheseLigne{\mathbb{R}^d}$ the Borel $\sigma$-field of $\mathbb{R}^d$, $\mathbb{M}\parentheseLigne{\mathbb{R}^d}$ the set of all Borel measurable functions $f$ on $\mathbb{R}^d$ and $\norm{\cdot}$ the Euclidean norm on $\mathbb{R}^d$.
For $\mu$ a probability measure on $\parentheseLigne{\Rd,\mathcal{B}\parentheseLigne{\Rd}}$ and $f \in \mathbb{M}\parentheseLigne{\mathbb{R}^d}$ a $\mu$-integrable function, denote by $\mu\parentheseLigne{f}$ the integral of $f$ with respect to (w.r.t.) $\mu$.
Let $\mu$ and $\nu$ be two sigma-finite measures on $\parentheseLigne{\Rd,\mathcal{B}\parentheseLigne{\Rd}}$. 
Denote by $\mu \ll \nu$ if $\mu$ is absolutely continuous w.r.t. $\nu$ and $\rmd \mu/\rmd \nu$ the associated density. 
We say that $\zeta$ is a transference plan of $\mu$ and $\nu$ if it is a probability measure on $\parentheseLigne{\Rd \times \Rd,\mathcal{B}\parentheseLigne{\Rd \times \Rd}}$ such that for all measurable set $\mathsf{A}$ of $\Rd$, $\zeta\parentheseLigne{\mathsf{A} \times \Rd} = \mu\parentheseLigne{\mathsf{A}}$
and $\zeta\parentheseLigne{\Rd \times \mathsf{A}} = \nu\parentheseLigne{\mathsf{A}}$.
We denote by $\mathcal{T}\parentheseLigne{\mu,\nu}$ the set of transference plans of $\mu$ and $\nu$.
 In addition, we say that a couple of $\mathbb{R}^d$-random variables $\parentheseLigne{X,Y}$ is a coupling of $\mu$ and $\nu$ if there exists $\zeta \in \mathcal{T}\parentheseLigne{\mu,\nu}$ such that $\parentheseLigne{X,Y}$ are distributed according to $\zeta$.
 We denote by $\mathcal{P}_{1}\parentheseLigne{\Rd}$ the set of probability measures with finite $1$-moment: for all $\mu \in \mathcal{P}_{1}\parentheseLigne{\Rd},\int_{\Rd} \|x\| \rmd\mu\parentheseLigne{x} < \infty$. 
 We denote by $\mathcal{P}_{2}\parentheseLigne{\Rd}$ the set of probability measures with finite $2$-moment: for all $\mu \in \mathcal{P}_{2}\parentheseLigne{\Rd},\int_{\Rd} \|x\|^{2} \rmd\mu\parentheseLigne{x} < \infty$. 
We define the squared Wasserstein distance of order $2$ associated with $\|\cdot\|$ for any probability measures $\mu,\nu \in \mathcal{P}_{2}\parentheseLigne{\Rd}$ by
\begin{equation}
\mathrm{W}_{2}^{2} \parentheseLigne{\mu,\nu} = \inf_{\zeta \in \mathcal{T}\parentheseLigne{\mu,\nu}} \int_{\mathbb{R}^d \times \mathbb{R}^d}\|x-y\|^{2}\mathrm{d}\zeta\parentheseLigne{x,y} \eqsp.
\end{equation}
By \citet[Theorem 4.1]{Villani2008}, for all $\mu$, $\nu$ probability measures on $\Rd$, there exists a transference plan $\zeta^{\star} \in \mathcal{T}\parentheseLigne{\mu,\nu}$ such that for any coupling $\parentheseLigne{X,Y}$ distributed according to $\zeta^{\star}$, $\mathrm{W}_{2}\parentheseLigne{\mu,\nu} = \E[\|x-y\|^{2}]^{1/2}$. 
This kind of transference plan (respectively coupling) will be called an optimal transference plan (respectively optimal coupling) associated with $W_{2}$. 
By \citet[Theorem 6.16]{Villani2008}, $\mathcal{P}_{2}\parentheseLigne{\Rd}$ equipped with the
Wasserstein distance $\mathrm{W}_{2}$ is a complete separable metric space.
For the sake of simplicity, with little abuse, we shall use the same notations for
a probability distribution and its associated probability density function.
For $n \ge 1$, we refer to the set of integers between $1$ and $n$ with the notation $[n]$.
The $d$-multidimensional Gaussian probability distribution with mean $\mu \in \Rd$ and covariance matrix $\Sigma \in \mathbb{R}^{d \times d}$ is denoted by $\gauss\parentheseLigne{\mu,\Sigma}$.
Given two matrices $M, N \in \R^{k \times d}$, the principal angle distance between the
subspaces spanned by the columns of $M$ and $N$ is given by $\mathrm{dist}(M,N) = \|\hat{M}_\perp^\dagger \hat{N}\|_2 = \|\hat{N}_\perp^\dagger \hat{M}\|_2$ where $\hat{M},\hat{N}$ are orthonormal bases of $\mathrm{Span}(M)$ and $\mathrm{Span}(N)$, respectively. Similarly, $\hat{M}_\perp,\hat{N}_\perp$ are orthonormal bases of orthogonal complements $\mathrm{Span}(M)^\perp$ and $\mathrm{Span}(N)^\perp$, respectively. This principal angle distance is upper bounded by 1, see \citet[Definition 1]{10.1145/2488608.2488693}.

\noindent\textbf{Outline.} 
This supplementary material aims at providing the interested reader with a further understanding of the statements pointed out in the main paper. 
More precisely, in \Cref{sec:insights}, we support the proposed methodology \ours with algorithmic and theoretical details.
In \Cref{sec:proof}, we prove the main results stated in the main paper.
Finally, in \Cref{sec:expe}, we provide further experimental design choices and show complementary numerical results.


\section{Algorithmic and Theoretical Insights}
\label{sec:insights}

In this section, we highlight alternative but limited ways to cope with feature space heterogeneity; and justify the usage, in the objective function \eqref{eq:PFL_ERM_reg} of the main paper, of Wasserstein distances with empirical probability distributions instead of true ones.
In addition, we detail the general steps depicted \Cref{algo:FLIC-FEDREP}.

\subsection{Some Limited but Common Alternatives to Cope with Feature Space Heterogeneity}

Depending on the nature of the spaces $\{\mathcal{X}_i\}_{i \in [b]}$, the feature transformation functions $\{\phi_i\}_{i \in [b]}$ can be either known beforehand or more difficult to find.
As an example, if for any $i \in [b]$, $\mathcal{X}_i \subseteq \mathcal{X}$, then we can set mask functions as feature transformation functions in order to only consider features that are shared across all the clients.
Besides, we could consider multimodal embedding models to perform feature transformation on each client \citep{multimodal_embedding}.
For instance, if clients own either pre-processed images or text of titles, descriptions and tags, then we can use the Contrastive Language-Image Pre-Training (CLIP) model as feature transformation function \citep{CLIP}. 
These two examples lead to the solving of a classical (personalised) FL problem which can be performed using existing state-of-the-art approaches.
However, when the feature transformation functions cannot be easily found beforehand, solving the FL problem at stake becomes more challenging and has never been addressed in the federated learning literature so far, up to the authors' knowledge. 

\subsection{Use of Wasserstein Losses Involving Empirical Probability Distributions}
\label{subsec:Wass}

Since the true probability distributions $\{\nu_{\phi_i}^{(c)}; c \in \mathcal{Y}_i\}_{i \in [b]}$ are unknown a priori, we propose in the main paper to estimate the latter using $\{\hat{\nu}_{\phi_i}^{(c)}; c \in \mathcal{Y}_i\}_{i \in [b]}$ and to replace $\mathrm{W}_2^2\pr{\mu_c,\nu_{\phi_i}^{(c)}}$ by $\mathrm{W}_2^2\pr{\mu_c,\hat{\nu}_{\phi_i}^{(c)}}$ in the objective function \eqref{eq:PFL_ERM_reg} in the main paper. 
As shown in the following result, this assumption is theoretically grounded when the marginal distributions of the input features are Gaussian.

\begin{theorem}
    \label{theorem:wass}
    For any $i \in [b]$ and $c \in [C]$, let $n_i^{(c)} = |\mathrm{D}_i^{(c)}|$ where $\mathrm{D}_i^{(c)}$ denotes the subset of the local data set $\mathrm{D}_i$ only involving observations associated to the label $c$. 
    Besides, assume that $\nu_{\phi_i}^{(c)}$ is Gaussian with mean vector $m_i^{(c)} \in \R^k$ and full-rank covariance matrix $\Sigma_i^{(c)} \in \R^{k \times k}$.
    Then, we have in the limiting case $n_i^{(c)} \rightarrow \infty$,
        \begin{align}
            \sqrt{n_i^{(c)}}\pr{\mathrm{W}_2^2\pr{\mu_c,\hat{\nu}_{\phi_i}^{(c)}} - \mathrm{W}_2^2\pr{\mu_c,\nu_{\phi_i}^{(c)}}} \overset{\mathrm{in \ distribution}}{\xrightarrow{\hspace{2cm}}} Z_i^{(c)} \eqsp,
        \end{align}
    where $Z_i^{(c)} \sim \mathrm{N}(0,s_i^{(c)})$ and $s_i^{(c)} = 4(m_i^{(c)} - v_c)^\top\Sigma_i^{(c)}(m_i^{(c)} - v_c) + 2\mathrm{Tr}(\Sigma_i^{(c)}\Sigma_c) - 4\sum_{j=1}^k \kappa_j^{\half} r_j^\top \Sigma_c^{-\half}\Sigma_i^{(c)}\Sigma_c^{\half}r_j$, with $\{\kappa_j,r_j\}_{j \in [k]}$ standing for (eigenvalue, eigenvector) pairs of the symmetric covariance matrix $\Sigma_i^{(c)}$.
\end{theorem}
\begin{proof}
    The proof follows from \citet[Theorem 2.1]{RIPPL201690} with the specific choices $\mu_1 = \nu_{\phi_i}^{(c)}$, $\mu_2 = \mu_c$ and $\hat{\mu}_1 = \hat{\nu}_{\phi_i}^{(c)}$ which are defined in \Cref{sec:methodo} in the main paper.
\end{proof}

\subsection{Detailed Pseudo-Code for \Cref{algo:FLIC-FEDREP}}

In \Cref{algo:FLIC-FEDREP_supp}, we provide algorithmic support to Algorithm 1 in the main paper by detailing how to perform each step.
Note that we use the decomposition $\Sigma = L L^\top$ to enfore the positive semi-definite constraint for the covariance matrix $\Sigma$.

\begin{algorithm}[H]
    \caption{Detailed version of \textcolor{magenta}{\texttt{FLIC}} when using \texttt{FedRep}}
    \label{algo:FLIC-FEDREP_supp}
    \begin{algorithmic}[1]
        \REQUIRE{ initialisation $\alpha^{(0)}$, $\mu_{1:C}^{(0)} = [\Sigma_{1:C}^{(0)}, v_{1:C}^{(0)}]$ with $\Sigma_c^{(0)} = L_c^{(0)} [L_c^{(0)}]^\top$, $\phi_{1:b}^{(0,0)}$, $\beta_{1:b}^{(0,0)}$ and step-size $\eta \leq \bar{\eta}$ for some $\bar{\eta} > 0$.}
        \FOR{$t=0$ {\bfseries to} $T-1$}
            \STATE Sample a set of $\mathsf{A}_{t+1}$ of active clients.
            \FOR{$i \in \mathsf{A}_{t+1}$} 
            \STATE The central server sends $\alpha^{(t)}$ and $\mu_{1:C}^{(t)}$ to $\mathsf{A}_{t+1}$.
            \STATE {\textcolor{darkgreen}{\textit{// Update local parameters}}}
            \FOR{$m=0$ {\bfseries to} $M-1$}
                \STATE Sample a fresh batch $\mathsf{I}_{t+1}^{(i,m)}$ of $n_i'$ samples with $n_i' \in [n_i]$.
                \STATE Sample $Z_c^{(j,t,m)} \sim \mu_c^{(t)}$ for $j \in \mathsf{I}_{t+1}^{(i,m)}$ and $c \in \mathcal{Y}_i$ via $Z_c^{(j,t,m)} = v_c^{(t)} + L_c^{(t)}\xi^{(t,m)}_i$ where $\xi^{(t,m)}_i \sim \mathrm{N}(0_k,\mathrm{I}_k)$.
                \STATE \small $\phi_i^{(t,m+1)} = \phi_i^{(t,m)} - \eta \displaystyle\frac{n_i}{|\mathsf{I}_{t+1}^{(i,m)}|} \sum_{j \in \mathsf{I}_{t+1}^{(i,m)}} \nabla_{\phi_i} \ell\pr{y_i^{(j)}, g_{[\alpha^{(t)},\beta_i^{(t,m)}]}^{(i)}\br{\phi_i^{(t,m)}\pr{x_i^{(j)}}}} - \eta \lambda_1 \sum_{c \in \mathcal{Y}_i} \nabla_{\phi_i} \mathrm{W}_2^2 \pr{\mu_c^{(t)},\nu_{\phi_i^{(t,m)}}^{(c)}}$.
                \STATE $\beta_i^{(t,m+1)} \leftarrow \beta_i^{(t,m)} - \eta \displaystyle\frac{n_i}{|\mathsf{I}_{t+1}^{(i,m)}|} \sum_{j \in \mathsf{I}_{t+1}^{(i,m)}} \bbr{\nabla_{\beta_i} \ell\pr{y_i^{(j)}, g_{[\alpha^{(t)},\beta_i^{(t,m)}]}^{(i)}\br{\phi_i^{(t,m)}\pr{x_i^{(j)}}}} - \eta \lambda_2 \sum_{c \in \mathcal{Y}_i} \nabla_{\beta_i} \ell\pr{y_i^{(j)}, g_{[\alpha^{(t)},\beta_i^{(t,m)}]}^{(i)}\br{Z_c^{(j,t,m)}}}}$.
            \ENDFOR
            \STATE $\phi_i^{(t+1,0)} = \phi_i^{(t,M)}$.
            \STATE $\beta_i^{(t+1,0)} = \beta_i^{(t,M)}$.
            
            \STATE {\textcolor{darkgreen}{\textit{// Update global parameters}}}
            \STATE \small $\alpha_i^{(t+1)} \leftarrow \alpha^{(t)} - \eta \displaystyle\frac{n_i}{|\mathsf{I}_{t+1}^{(i,M)}|} \sum_{j \in \mathsf{I}_{t+1}^{(i,M)}} \bbr{\nabla_{\alpha} \ell\pr{y_i^{(j)}, g_{[\alpha^{(t)},\beta_i^{(t,M)}]}^{(i)}\br{\phi_i^{(t,M)}\pr{x_i^{(j)}}}} - \eta \lambda_2 \sum_{c \in \mathcal{Y}_i} \nabla_{\alpha} \ell\pr{y_i^{(j)}, g_{[\alpha^{(t)},\beta_i^{(t,M)}]}^{(i)}\br{Z_c^{(j,t,M)}}}}$.
            \FOR{$c=1$ {\bfseries to} $C$}
                \STATE Update $\hat{m}_i^{(c,t)},\hat{\Sigma}_i^{(c,t)}$ using $\phi_i^{(t,M)}$. 
                \STATE $v_{i,c}^{(t+1)} = v_{c}^{(t)} - \eta \lambda_1 \nabla_{v_c} \norm{v_c^{(t)} - \hat{m}_i^{(c,t)}}^2 - \eta \lambda_2 \sum_{c \in \mathcal{Y}_i} \displaystyle\frac{n_i}{|\mathsf{I}_{t+1}^{(i,m)}|} \sum_{j \in \mathsf{I}_{t+1}^{(i,m)}}\nabla_{v_c} \ell\pr{y_i^{(j)}, g_{[\alpha^{(t)},\beta_i^{(t,M)}]}^{(i)}\br{Z_c^{(j,t,M)}}}$.
                \STATE $L_{i,c}^{(t+1)} = L_{c}^{(t)} - \eta \lambda_1 \nabla_{L_c} \mathfrak{B}^2\pr{L_c^{(t)}[L_c^{(t)}]^\top, \hat{\Sigma}_i^{(c,t)}} - \eta \lambda_2 \sum_{c \in \mathcal{Y}_i} \displaystyle\frac{n_i}{|\mathsf{I}_{t+1}^{(i,m)}|} \sum_{j \in \mathsf{I}_{t+1}^{(i,m)}}\nabla_{L_c} \ell\pr{y_i^{(j)}, g_{[\alpha^{(t)},\beta_i^{(t,M)}]}^{(i)}\br{Z_c^{(j,t,M)}}}$.
            \ENDFOR    
            \STATE {\textcolor{darkgreen}{\textit{//~Communication with the server}}}
            \STATE Send $\alpha_i^{(t+1)}$, $v_{i,1:C}^{(t+1)}$ and $L_{i,1:C}^{(t+1)}$ to central server.
        \ENDFOR
        \STATE {\textcolor{darkgreen}{\textit{//~Averaging global parameters}}}
        \STATE $\alpha^{(t+1)} = \frac{b}{|\mathsf{A}_{t+1}|}\sum_{i \in \mathsf{A}_{t+1}} w_i \alpha_i^{(t+1)}$.
        \FOR{$c=1$ {\bfseries to} $C$}
            \STATE $v_c^{(t+1)} = (b/|\mathsf{A}_{t+1}|) \sum_{i \in \mathsf{A}_{t+1}} \omega_i v_{i,c}^{(t+1)}$. 
            \STATE $L_c^{(t+1)} = (b/|\mathsf{A}_{t+1}|) \sum_{i \in \mathsf{A}_{t+1}} \omega_i L_{i,c}^{(t+1)}$ and set $\Sigma_c^{(t+1)} = L_c^{(t+1)} [L_c^{(t+1)}]^\top$.
        \ENDFOR
        
        \ENDFOR
        \ENSURE parameters $\alpha^{(T)}$, $\mu_{1:C}^{(T)}$, $\phi_{1:b}^{(T,0)}$, $\beta_{1:b}^{(T,0)}$. 

\end{algorithmic}
\end{algorithm}

\subsection{Additional Algorithmic Insights}

\noindent \textbf{Scalability.} When the number of classes $C$ is large, both local computation and communication costs are increased. 
In this setting, we propose to partition all the classes into $C_{\mathrm{meta}} \ll C$ meta-classes and consider reference measures $\{\mu_c\}_{c \in [C_{\mathrm{meta}}]}$ associated to these meta-classes.
As an example, if we are considering a dataset made of features associated to animals, the meta-class refers to an animal (\emph{e.g.} a dog) and the class refers to a specific breed (\emph{e.g.} golden retriever).

\noindent \textbf{Privacy Consideration.} As other standard (personalised) FL algorithms, \ours satisfies first-order privacy guarantees by not allowing raw data exchanges but rather exchanges of local Gaussian statistics. 
Note that \ours stands for a post-hoc approach and can be combined with other privacy/confidentiality techniques such as differential privacy \citep{10.1561/0400000042}, secure aggregation via secure multi-party computation \citep{pmlr-v162-chen22c} or trusted execution environments \citep{10.1145/3458864.3466628}.

\noindent \textbf{Inference on New Clients.} When a client who has not participated to the training procedure appears, there is no need to re-launch a potentially costly federated learning procedure.
Instead, the server sends the shared parameters $\{\alpha^{(T)},\mu_{1:C}^{(T)}\}$ to the new client and the latter only needs to learn the local parameters $\{\phi_i,\beta_i\}$.

\section{Proof of \Cref{theorem2}}
\label{sec:proof}

This section aims at proving \Cref{theorem2} in the main paper. 
To this end, we first provide in \Cref{subsec:proj_feat} a closed-form expression for the estimated embedded features based on the features embedded by the oracle. 
Then, in \Cref{subsec:lemmata}, we show technical lemmata that will be used in \Cref{subsec:proof_thm2} to show \Cref{theorem2}.

To prove our results, we consider the following set of assumptions.

\begin{assumption}
    \label{ass}
    \begin{enumerate}[wide, labelwidth=!, labelindent=0pt,label=(\roman*),noitemsep,nolistsep]

        \item For any $i \in [b]$, $j \in [n_i]$, ground-truth embedded features $\phi_i^\star(x_i^{(j)})$ are distributed according to $\mathrm{N}(0_k,\mathrm{I}_k)$.\label{ass:1}
        \item Ground-truth model parameters satisfy $\|\beta_i^\star\|_2 = \sqrt{d}$ for $i \in [b]$ and $A^\star$ has orthonormal columns.\label{ass:2}
        \item For any $t \in \{0,\ldots,T-1\}, |\mathsf{A}_{t+1}| = \floor{rb}$ with $1 \leq \floor{rb} \leq b$, and if we select $\floor{rb}$ clients, their ground-truth head parameters $\{\beta_i^\star\}_{i \in \mathsf{A}_{t+1}}$ span $\R^d$.\label{ass:3}
        \item In (2) in the main paper, $\ell(\cdot,\cdot)$ is the $\ell_2$ norm, $\omega_i=1/b$, $\theta_i = [A,\beta_i]$ and $g_{\theta_i}^{(i)}(x) = (A\beta_i)^\top x$ for $x \in \R^k$.
    \end{enumerate}
\end{assumption}

\subsection{Estimation of the Feature Transformation Functions}
\label{subsec:proj_feat}

As in Section 4 in the main paper, we assume that $x_i^{(j)} \sim \mathrm{N}(m_i,\Sigma_i)$ with $m_i \in \R^{k_i}$ and $\Sigma_i \in \R^{k_i \times k_i}$ for $i \in [b], j \in [n_i]$.
In addition, we consider that the continuous scalar labels are generated via the oracle model $y_i^{(j)} = (A^\star\beta^\star_i)^\top \phi_i^\star(x_i^{(j)})$ where $A^\star \in \R^{k \times d}$, $\beta_i^\star \in \R^d$ and $\phi_i^\star(\cdot)$ are ground-truth parameters and feature transformation function, respectively. 
Under \Cref{ass}-\ref{ass:1}, the oracle feature transformation functions $\{\phi^\star_i\}_{i \in [b]}$ are assumed to map $k_i$-dimensional Gaussian distributions $\mathrm{N}(m_i,\Sigma_i)$ to a common $k$-dimension Gaussian $\mathrm{N}(0_k,\mathrm{I}_k)$.
As shown in \citet[Theorem 4.1]{delon_desolneux_salmona_2022}, there exist closed-form expressions for $\{\phi^\star_i\}_{i \in [b]}$, which can be shown to stand for solutions of a Gromov-Wasserstein problem restricted to Gaussian transport plans. 
More precisely, these oracle feature transformation stand for affine maps and are of the form, for any $i \in [b]$,
\begin{equation}
    \phi_i^\star \pr{x_i^{(j)}} = \br{\tilde{I}_k^{(i,\star)} (D_i^{(k)})^{-\half} \quad 0_{k,k_i-k}}\pr{x_i^{(j)} - m_i}\eqsp,
\end{equation}
where $\tilde{I}_k^{(i,\star)} = \mathrm{diag}_k(\pm 1)$ is a $k$-dimensional diagonal matrix with diagonal elements in $\{-1,1\}$, $\Sigma_i = P_i D_i P_i^\top$ is the diagonalisation of $\Sigma_i$ and $D_i^{(k)}$ stands for the restriction of $D_i$ to the first $k$ components.
In the sequel, we assume that all oracle feature transformation functions share the same randomness, that is $\tilde{I}_k^{(i,\star)} = \tilde{I}_k^{\star} = \mathrm{diag}_k(\pm 1)$. 

For the sake of simplicity, we assume that we know the true latent distribution of $\phi_i^\star (x_i^{(j)})$ and as such consider a pre-fixed reference latent distribution that equals the latter, that is $\mu = \mathrm{N}(0_k,\mathrm{I}_k)$.
Since we know from \citet[Theorem 4.1]{delon_desolneux_salmona_2022} that there exist mappings between Gaussian distributions with supports associated to different metric spaces, we propose an estimate for the ground-truth feature transformation functions defined by for any $i \in [b]$,
\begin{equation}
    \hat{\phi_i}\pr{x_i^{(j)}} = \br{\tilde{I}_k (D_i^{(k)})^{-\half} \quad 0_{k,k_i-k}}\pr{x_i^{(j)} - m_i}\eqsp,
\end{equation}
where $\tilde{I}_k = \mathrm{diag}_k(\pm 1)$.
By noting that $\tilde{I}_k = Q \tilde{I}_k^{\star}$, where $Q \in \R^{k \times k}$ is a diagonal matrix of the form $\mathrm{diag}_k(\pm 1)$, it follows that 
\begin{equation}
    \hat{\phi_i}\pr{x_i^{(j)}} = Q\phi_i^\star \pr{x_i^{(j)}}\eqsp.\label{eq:1}
\end{equation}

In \Cref{subsec:proof_thm2}, the equation \eqref{eq:1} will allow us to relate the ground-truth labels $y_i^{(j)} = (A^\star\beta^\star_i)^\top \phi_i^\star(x_i^{(j)})$ with estimated predictions $\hat{y}_i^{(j)} = (A^{(T)}\beta^{(T)}_i)^\top \hat{\phi}_i(x_i^{(j)})$ via \Cref{algo:FLIC-FedRep-linreg} starting from the same embedded features.

\subsection{Proof of Theorem 1}
\label{subsec:proof_thm2}

\begin{algorithm}[t]
    \caption{\textcolor{magenta}{\texttt{FLIC-FedRep} for linear regression and Gaussian features}}
    \label{algo:FLIC-FedRep-linreg}
    \begin{algorithmic}[1]
        \REQUIRE{step size $\eta$, number of outer iterations $T$, participation rate $r \in (0,1)$, diagonalizations $\Sigma_i = P_i D_i P_i^\top$ sorting eigenvalues in decreasing order.}
        \STATE {\textcolor{darkgreen}{\textit{// Estimation of embedded features}}}
        \STATE For each client $i\in [b]$, set $\hat{\phi}_i\pr{x_i^{(j)}} = \br{\tilde{I}_k (D_i^{(k)})^{-\half} \quad 0_{k,k_i-k}}\pr{x_i^{(j)} - m_i}$.
        \STATE {\textcolor{darkgreen}{\textit{// Initialisation $A^{(0)}$}}}
        \STATE Each client $i\in [b]$ sends $Z_i = (1/n_i)\sum_{j=1}^{n_i} (y_i^{(j)})^2 \hat{\phi}_i\pr{x_i^{(j)}}[\hat{\phi}_i\pr{x_i^{(j)}}]^\top$ to the central server.
        \STATE The central server computes $U D U^\top \leftarrow \mathrm{rank-}d \ \mathrm{SVD}\pr{(1/b)\sum_{i=1}^b Z_i}$.
        \STATE The central server initialises $A^{(0)} = U$.
        \FOR{$t=0$ {\bfseries to} $T-1$}
            \STATE Sample a set of $\mathsf{A}_{t+1}$ of active clients such that $|\mathsf{A}_{t+1}| = \floor{rb}$.
            \FOR{$i \in \mathsf{A}_{t+1}$} 
            \STATE The central server sends $A^{(t)}$to $\mathsf{A}_{t+1}$.
            \STATE {\textcolor{darkgreen}{\textit{// Update local parameters}}}
            \STATE $\beta_i^{(t+1)} = \argmin_{\beta_i} \sum_{j=1}^{n_i} \pr{y_i^{(j)} - \beta_i^\top [A^{(t)}]^\top \hat{\phi}_i\pr{x_i^{(j)}}}^2$.
            
            \STATE {\textcolor{darkgreen}{\textit{// Update global parameters}}}
            \STATE $A_i^{(t+1)} = A^{(t)} - \eta \nabla_{A} \sum_{j=1}^{n_i} \pr{y_i^{(j)} - [\beta_i^{(t+1)}]^\top A^\top \hat{\phi}_i\pr{x_i^{(j)}}}^2$.

            \STATE {\textcolor{darkgreen}{\textit{//~Communication with the server}}}
            \STATE Send $A_i^{(t+1)}$ to the central server.
        \ENDFOR
        \STATE {\textcolor{darkgreen}{\textit{//~Averaging and orthogonalisation of global parameter}}}
        \STATE $\bar{A}^{(t+1)} = \frac{1}{\floor{rb}}\sum_{i \in \mathsf{A}_{t+1}} A_i^{(t+1)}$.
        \STATE $A^{(t+1)}, R^{(t+1)} \leftarrow \mathrm{QR}\pr{\bar{A}^{(t+1)}}$.
        \ENDFOR
        \ENSURE parameters $A^{(T)}$, $\beta_{1:b}^{(T)}$. 

\end{algorithmic}
\end{algorithm}

Let $B \in \R^{b \times d}$ the matrix having local model parameters $\{\beta_i\}_{i \in [b]}$ as columns and denote by $B_{\mathsf{A}_{t+1}} \in \R^{\floor{rb} \times d}$ its restriction to the row set defined by $\mathsf{A}_{t+1}$ where $|\mathsf{A}_{t+1}| = \floor{rb}$ for some $r\in(0,1]$.
For the sake of simplicity, we assume in the sequel that all clients have the same number of data points that is for any $i \in [b]$, $n_i=n$.
For random batches of samples $\{(x_i^{(j)},y_i^{(j)}), j \in [n]\}_{i \in [\floor{rb}]}$, we define similarly to \citet{pmlr-v139-collins21a,10.1145/2488608.2488693}, the random linear operator $\mathcal{A}: \R^{\floor{rb} \times d} \rightarrow \R^{\floor{rb}n}$ for any $M \in \R^{\floor{rb} \times d}$ as $\mathcal{A}(M) = [\langle e_i(\phi_i^\star(x_i^{(j)}))^\top, M \rangle]_{1 \leq i \leq \floor{rb}, 1 \leq j \in n}$, where $e_i$ stands for the $i$-th standard vector of $\R^{\floor{rb}}$.
Using these notations, it follows from \Cref{algo:FLIC-FedRep-linreg} that for any $t \in \{0,\ldots,T-1\}$, the model parameters $\theta_i^{(t+1)} = [A^{(t+1)},\beta_i^{(t+1)}]$ are computed as follows:

\begin{align}
    &B_{\mathsf{A}_{t+1}}^{(t+1)} = \argmin_{B_{\mathsf{A}_{t+1}}} \frac{1}{\floor{rb} n} \norm{\mathcal{A}^{(t+1)}\pr{B^\star_{\mathsf{A}_{t+1}} [A^\star]^\top - B_{\mathsf{A}_{t+1}} [A^{(t)}]^\top Q}}^2 \eqsp,\label{eq:B}\\
    &\bar{A}^{(t+1)} = \bar{A}^{(t)} -  \frac{\eta}{\floor{rb} n} \br{(\mathcal{A}^{(t+1)})^\dagger\mathcal{A}^{(t+1)}\pr{B^\star_{\mathsf{A}_{t+1}} [A^\star]^\top - B_{\mathsf{A}_{t+1}}^{(t+1)} [A^{(t)}]^\top Q}}^\top Q B_{\mathsf{A}_{t+1}}^{(t+1)}\eqsp,\\
    &A^{(t+1)}, R^{(t+1)} \leftarrow \mathrm{QR}\pr{\bar{A}^{(t+1)}} \eqsp,\label{eq:A}
\end{align}
where $\mathcal{A}^{(t+1)}$ stands for a specific instance of $\mathcal{A}$ depending on the random subset of active clients available at each round and $\mathcal{A}^\dagger$ is the adjoint operator of $\mathcal{A}$ defined by $\mathcal{A}^\dagger(M) = \sum_{i \in [\floor{rb}]} \sum_{i=1}^n [\langle e_i(\phi_i^\star(x_i^{(j)}))^\top, M \rangle] e_i(\phi_i^\star(x_i^{(j)}))$.

The update in \eqref{eq:B} admits a closed-form expression as shown in the following lemma. 

\begin{lemma}
    \label{lemma:lemma_recursionB}
    For any $t \in \ldots{0,\ldots,T-1}$, we have
    \begin{align}
        B^{(t+1)}_{\mathsf{A}_{t+1}} = B^\star_{\mathsf{A}_{t+1}} [A^\star]^\top Q A^{(t)} - F^{(t)}\eqsp,
    \end{align}
    where $F^{(t)}$ is defined in \eqref{eq:F}, $A^{(t)}$ is defined in \eqref{eq:A} and $B^{(t)}_{\mathsf{A}_t}$ is defined in \eqref{eq:B}.
\end{lemma}
\begin{proof}  
    The proof follows from the same steps as in \citet[Proof of Lemma 1]{pmlr-v139-collins21a} using \eqref{eq:B}.
\end{proof}

Under \Cref{ass}, we have the following non-asymptotic convergence result.

\begin{theorem}
    \label{theorem:1}
    Assume \Cref{ass}.
    Then, for any $x_i \in \R^{k_i}$, we have $\hat{\phi}_i(x_i) = Q \phi_i^\star(x_i)$ where $Q \in \R^{k \times k}$ is of the form $\mathrm{diag}_k(\pm 1)$.
    Define $E_0 = \mathrm{dist}(A^{(0)},Q A^\star)$.
    Assume that $n \geq c(d^3\log(\floor{rb}))/E_0^2 + d^2 k / (E_0^2 \floor{rb})$ for some absolute constant $c > 0$. 
    Then, for any $t \in \{0,\ldots,T-1\}$, $\eta \leq 1/(4\bar{\sigma}^2_{\max,\star})$ and with high probability at least $1 - e^{-110k} - e^{-110d^2 \log(\floor{rb})}$, we have
    $$
    \mathrm{dist}(A^{(t+1)},QA^\star) \leq (1 - \kappa)^{(t+1)/2} \mathrm{dist}(A^{(0)},QA^\star)\eqsp,
    $$
    where $A^{(t)}$ is computed via \Cref{algo:FLIC-FedRep-linreg}, $\mathrm{dist}$ denotes the principal angle distance and $\kappa \in (0,1)$ is defined as
    $$
    \kappa = 1 - \eta E_0 \bar{\sigma}^2_{\min,\star} / 2.
    $$
\end{theorem}

\begin{proof}
    The proof follows first by plugging \Cref{lemma:C1}, \Cref{lemma:C2} and \Cref{lemma:R} into \Cref{lemma:control}. 
    Then, we use the same technical arguments and steps as in \citet[Proof of Lemma 6]{pmlr-v139-collins21a}.
\end{proof}

\subsection{Technical Lemmata}
\label{subsec:lemmata}

In this section, we provide a set of useful technical lemmata to prove our main result in \Cref{subsec:proof_thm2}.

\noindent \textbf{Notations.} 
We begin by defining some notations that will be used in the sequel.
For any $t \in \{0,\ldots,T-1\}$, we define

\begin{equation}
    \label{eq:Z}
    Z^{(t+1)} = B_{\mathsf{A}_{t+1}}^{(t+1)} [A^{(t)}]^\top Q - B^\star_{\mathsf{A}_{t+1}} [A^\star]^\top \eqsp.
\end{equation}

In addition, let

\begin{align}
    G^{(t)} =
  \left[ {\begin{array}{cccc}
    G_{11}^{(t)} & \cdots & G_{1d}^{(t)}\\
    \vdots & \ddots & \vdots\\
    G_{d1}^{(t)} & \cdots & G_{dd}^{(t)}\\
  \end{array} } \right],
  C^{(t)} =
  \left[ {\begin{array}{cccc}
    C_{11}^{(t)} & \cdots & C_{1d}^{(t)}\\
    \vdots & \ddots & \vdots\\
    C_{d1}^{(t)} & \cdots & C_{dd}^{(t)}\\
  \end{array} } \right],
  D^{(t)} =
  \left[ {\begin{array}{cccc}
    D_{11}^{(t)} & \cdots & D_{1d}^{(t)}\\
    \vdots & \ddots & \vdots\\
    D_{d1}^{(t)} & \cdots & D_{dd}^{(t)}\\
  \end{array} } \right]\eqsp,
\end{align}
where for $p,q \in [d]$, 
\begin{align}
    G_{pq}^{(t)} &= \frac{1}{n} \sum_{i \in \mathsf{A}_{t+1}} \sum_{j=1}^n e_i\pr{\phi_i^\star(x_i^{(j)})}^\top Q a_p^{(t)} [a_q^{(t)}]^\top Q \phi_i^\star(x_i^{(j)}) e_i^\top\eqsp,\label{eq:G}\\
    C_{pq}^{(t)} &= \frac{1}{n} \sum_{i \in \mathsf{A}_{t+1}} \sum_{j=1}^n e_i\pr{\phi_i^\star(x_i^{(j)})}^\top Q a_p^{(t)} [a_q^\star]^\top Q \phi_i^\star(x_i^{(j)}) e_i^\top\eqsp,\label{eq:C}\\
    D_{pq}^{(t)} &= \langle a_p^{(t)}, a_q^\star \rangle \mathrm{I}_{\floor{rb}}\eqsp,\label{eq:D}
\end{align}
with $a_p^{(t)} \in \R^{k}$ standing for the $p$-th column of $A^{(t)} \in \R^{k \times d}$; and $a_p^\star \in \R^{k}$ standing for the $p$-th column of $A^\star \in \R^{k \times d}$.
Finally, we define for any $i \in \mathsf{A}_{t+1}$, 
\begin{align}
    &\Pi^i = \frac{1}{n}\sum_{j=1}^n \phi_i^\star(x_i^{(j)})[\phi_i^\star(x_i^{(j)})]^\top\eqsp,\label{eq:Pi}\\
    &(G^{(t)})^i = [A^{(t)}]^\top Q \Pi^i Q A^{(t)} \eqsp, \label{eq:Gi}\\
    &(C^{(t)})^i = [A^{(t)}]^\top Q \Pi^i Q A^\star \eqsp,\label{eq:Ci}\\
    &(D^{(t)})^i = [A^{(t)}]^\top Q A^\star \eqsp.\label{eq:Di}
\end{align}

Using these notations, we also define $\tilde{\beta}^\star = [(\beta^\star_1)^\top, \ldots, (\beta^\star_d)^\top]^\top \in \R^{\floor{rb}d}$ and
\begin{align}
    F^{(t)} = [([G^{(t)}]^{-1}
(G^{(t)}D^{(t)} - C^{(t)})\tilde{\beta}^\star)_1, \ldots, ([G^{(t)}]^{-1}
(G^{(t)}D^{(t)} - C^{(t)})\tilde{\beta}^\star)_d] \eqsp.\label{eq:F}
\end{align}

\noindent \textbf{Technical results.}
To prove our main result in \Cref{theorem:1}, we begin by providing a first upper bound on the quantity of interest namely $\mathrm{dist}\pr{A^{(t+1)},QA^\star}$.
This is the purpose of the next lemma.

\begin{lemma}
    \label{lemma:control}
    For any $t \in \{0,\ldots,T-1\}$ and $\eta > 0$, we have
    $$
    \mathrm{dist}\pr{A^{(t+1)},QA^\star} \leq C_1 + C_2, \eqsp,
    $$
    where 
    \begin{align}
        &C_1 =  \norm{[A^\star_\perp]^\top Q A^{(t)} \pr{\mathrm{I}_d - \frac{\eta}{\floor{rb}}[B^{(t+1)}_{\mathsf{A}_{t+1}}]^\top B^{(t+1)}_{\mathsf{A}_{t+1}}}}_2 \norm{\pr{R^{(t+1)}}^{-1}}_2\eqsp,\label{eq:C1}\\
        & C_2 = \frac{\eta}{\floor{rb}} \norm{\pr{\frac{1}{n}[A^\star_\perp]^\top (Q\mathcal{A}^{(t+1)})^\dagger\mathcal{A}^{(t+1)}\pr{Z^{(t+1)}}Q - Z^{(t+1)}}^\top B^{(t+1)}_{\mathsf{A}_{t+1}}}_2 \norm{\pr{R^{(t+1)}}^{-1}}_2\eqsp,\label{eq:C2}
    \end{align}
    where $A^{(t)}$ is defined in \eqref{eq:A}, $B^{(t)}_{\mathsf{A}_t}$ is defined in \eqref{eq:B}, $Z^{(t)}$ is defined in \eqref{eq:Z} and $R^{(t)}$ comes from the QR factorisation of $\bar{A}^{(t)}$, see step 20 in \Cref{algo:FLIC-FedRep-linreg}.
    
\end{lemma}

\begin{proof}
    The proof follows from the same steps as in \citet[Proof of Lemma 6]{pmlr-v139-collins21a} and by noting that $\mathrm{dist}(A^{(t)},QA^\star) = \mathrm{dist}(QA^{(t)},A^\star)$ for $t \in \{0,\ldots,T-1\}$.
\end{proof}

We now have to control the terms $C_1$ and $C_2$. 
For the sake of clarity, we split technical results aiming to upper bound of $C_1$ and $C_2$ in two different paragraphs.

\noindent \textbf{Control of $C_1$.}

\begin{lemma}
    \label{lemma:C1}
    Assume \Cref{ass}.
    Let $\delta_d = c d^{3/2}\sqrt{\log(\floor{rb})} / n^{\half}$ for some absolute constant $c > 0$.
    Then, for any $t \in \{0,\ldots,T-1\}$, with probability at least $1 - e^{-111k^2\log(\floor{rb})}$, we have for $\delta_d \leq 1/2$ and $\eta \leq 1/(4 \bar{\sigma}^2_{\mathrm{max},\star})$
    $$ 
    C_1 \leq \br{\leq 1 - \eta \pr{1 - \mathrm{dist}\pr{A^{(0)},QA^\star}} \bar{\sigma}^2_{\mathrm{min},\star} + 2 \eta \frac{\delta_d}{1-\delta_d}\bar{\sigma}^2_{\mathrm{max}}} \ \mathrm{dist}\pr{A^{(t)},QA^\star}\norm{\pr{R^{(t+1)}}^{-1}}_2 \eqsp,
    $$
    where $\bar{\sigma}^2_{\mathrm{min}}, \bar{\sigma}^2_{\mathrm{max}}$ are defined in \eqref{eq:sigmin}-\eqref{eq:sigmax}, $C_1$ is defined in \eqref{eq:C1}, $A^{(t)}$ is defined in \eqref{eq:A} and $R^{(t)}$ comes from the QR factorisation of $\bar{A}^{(t)}$, see step 20 in \Cref{algo:FLIC-FedRep-linreg}.
\end{lemma}
\begin{proof}
    Using Cauchy-Schwarz inequality, we have 
    \begin{align}
    C_1 &\leq \norm{(A_\perp^\star)^\top Q A^{(t)}}_2 \norm{\mathrm{I}_d - \frac{\eta}{\floor{rb}}[B^{(t+1)}_{\mathsf{A}_{t+1}}]^\top B^{(t+1)}_{\mathsf{A}_{t+1}}}_2 \norm{\pr{R^{(t+1)}}^{-1}}_2 \\
    &= \mathrm{dist}\pr{A^{(t)},QA^\star} \norm{\mathrm{I}_d - \frac{\eta}{\floor{rb}}[B^{(t+1)}_{\mathsf{A}_{t+1}}]^\top B^{(t+1)}_{\mathsf{A}_{t+1}}}_2 \norm{\pr{R^{(t+1)}}^{-1}}_2\eqsp.
    \end{align}
    Define the following minimum and maximum singular values:
    \begin{align}
        \bar{\sigma}^2_{\mathrm{min},\star} &= \min_{\mathsf{A} \subseteq [b], |\mathsf{A}| = \floor{rb}} \sigma_{\mathrm{min}}\pr{\frac{1}{\sqrt{\floor{rb}}}B^\star_{\mathsf{A}}} \label{eq:sigmin}\\
        \bar{\sigma}^2_{\mathrm{max},\star} &= \min_{\mathsf{A} \subseteq [b], |\mathsf{A}| = \floor{rb}} \sigma_{\mathrm{max}}\pr{\frac{1}{\sqrt{\floor{rb}}}B^\star_{\mathsf{A}}} \label{eq:sigmax}\eqsp.
    \end{align}
    Using \citet[Proof of Lemma 6, equations (67)-(68)]{pmlr-v139-collins21a}, we have for $\delta_d \leq 1/2$ where $\delta_d$ is defined in \Cref{lemma:G} and $\eta \leq 1/(4 \bar{\sigma}^2_{\mathrm{max},\star})$,
    $$
    \norm{\mathrm{I}_d - \frac{\eta}{\floor{rb}}[B^{(t+1)}_{\mathsf{A}_{t+1}}]^\top B^{(t+1)}_{\mathsf{A}_{t+1}}}_2 \leq 1 - \eta \pr{1 - \mathrm{dist}\pr{A^{(0)},QA^\star}} \bar{\sigma}^2_{\mathrm{min},\star} + 2 \eta \frac{\delta_d}{1-\delta_d}\bar{\sigma}^2_{\mathrm{max},\star}\eqsp,
    $$
    with probability at least $1 - e^{-111k^2\log(\floor{rb})}$
    The proof is concluded by combining the two previous bounds.
\end{proof}

\noindent \textbf{Control of $C_2$.}
We begin by showing four intermediary results gathered in the next four lemmata.

\begin{lemma}
    \label{lemma:G}
    Assume \Cref{ass}.
    Let $\delta_d = c d^{3/2}\sqrt{\log(\floor{rb})} / n^{\half}$ for some absolute constant $c > 0$.
    Then, for any $t \in \{0,\ldots,T-1\}$, with probability at least $1 - e^{-111k^3\log(\floor{rb})}$, we have
    $$ 
    \norm{[G^{(t)}]^{-1}}_2\leq \frac{1}{1-\delta_d}\eqsp,
    $$
    where $G^{(t)}$ is defined in \eqref{eq:G}.
\end{lemma}
\begin{proof}
    The proof stands as a straightforward extension of \citet[Proof of Lemma 2]{pmlr-v139-collins21a} by noting that the random variable $Q\phi^\star_i(x_i^{(j)}) = \hat{\phi}_i(x_i^{(j)})$ is sub-Gaussian under \Cref{ass}-\ref{ass:1}; and as such is omitted.
\end{proof}

\begin{lemma}
    \label{lemma:GB}
    Assume \Cref{ass}.
    Let $\delta_d = c d^{3/2}\sqrt{\log(\floor{rb})} / n^{\half}$ for some absolute constant $c > 0$.
    Then, for any $t \in \{0,\ldots,T-1\}$, with probability at least $1 - e^{-111k^2\log(\floor{rb})}$, we have
    $$ 
    \norm{(G^{(t)}D^{(t)} - C^{(t)})B^\star_{\mathsf{A}_t}}_2\leq \delta_d \ \norm{B^\star_{\mathsf{A}_t}}_2 \ \mathrm{dist}\pr{A^{(t)},Q A^\star}\eqsp,
    $$
    where $G^{(t)}$ is defined in \eqref{eq:G}, $D^{(t)}$ is defined in \eqref{eq:D}, $C^{(t)}$ is defined in \eqref{eq:C} and $A^{(t)}$ in \eqref{eq:A}. 
\end{lemma}
\begin{proof}
    Without loss of generality and to ease notation, we remove the superscript $(t)$ in the proof and re-index the indexes of clients in $\mathsf{A}_{t+1}$.
    Let $H = GD -C$. From \eqref{eq:Pi}, \eqref{eq:Gi}, \eqref{eq:Ci} and \eqref{eq:Di}, it follows, for any $i \in [\floor{rb}]$, that
    \begin{align}
        H^i 
        &= G^iD^i - C^i = A^\top Q \Pi^i Q(A A^\top - \mathrm{I}_k)Q A^\star \eqsp.
    \end{align}
    Hence, by using the definition of $H$, we have
    \begin{align}
        \norm{(GD -C)\beta^\star}^2_2
        &= \sum_{i=1}^{\floor{rb}}\norm{H^i\beta_i^\star}^2_2 \leq \sum_{i=1}^{\floor{rb}}\norm{H^i}_2^2 \norm{\beta_i^\star}^2 \leq \frac{d}{\floor{rb}} \norm{B^\star}^2_2\sum_{i=1}^{\floor{rb}}\norm{H^i}_2^2\eqsp,
    \end{align}
    where the last inequality follows almost surely from \Cref{ass}-\ref{ass:3}.
    As in \citet[Proof of Lemma 3]{pmlr-v139-collins21a}, we then define for any $j \in [n]$, the vectors
    \begin{align}
        &u_i^{(j)} = \frac{1}{\sqrt{n}}[A^\star]^\top(A A^\top - \mathrm{I}_k)Q \phi_i^\star(x_i^{(j)}) \eqsp,\\
        &v_i^{(j)} = \frac{1}{\sqrt{n}}A^\top Q \phi_i^\star(x_i^{(j)})\eqsp.
    \end{align}
    Let $\mathcal{S}^{d-1}$ denotes the $d$-dimensional unit spheres. 
    Then, by \citet[Corollary 4.2.13]{vershyninHighdimensionalProbabilityIntroduction2018}, we can define $\mathcal{N}_d$, the $1/4$-net over $\mathcal{S}^{d-1}$ such that $|\mathcal{N}_d|\leq 9^d$.
    Therefore, by using \citet[Equation (4.13)]{vershyninHighdimensionalProbabilityIntroduction2018}, we have
    \begin{align}
        \norm{H^i}_2^2 
        &\leq 2 \max_{z,y \in \mathcal{N}_d} \sum_{j=1}^n \langle z,u_i^{(j)} \rangle \langle v_i^{(j)},y \rangle \eqsp.
    \end{align} 
    Since $\phi_i^\star(x_i^{(j)})$ is a standard Gaussian vector, it is sub-Gaussian and therefore $\langle z,u_i^{(j)} \rangle$ and $\langle v_i^{(j)},y \rangle$ are sub-Gaussian with norms $\|\frac{1}{\sqrt{n}}[A^\star]^\top(A A^\top - \mathrm{I}_k)Q\|_2 = (1/\sqrt{n})\mathrm{dist}(A,QA^\star)$ and $(1/\sqrt{n})$, respectively.
    In addition, we have 
    \begin{align}
        \mathbb{E}\br{\langle z,u_i^{(j)} \rangle \langle v_i^{(j)},y \rangle} 
        &= \frac{1}{n} \mathbb{E}\br{z^\top \frac{1}{\sqrt{n}}[A^\star]^\top(A A^\top - \mathrm{I}_k)Q \phi_i^\star(x_i^{(j)})[\phi_i^\star(x_i^{(j)})]^\top Q A y}  \\
        &= \frac{1}{n} z^\top \frac{1}{\sqrt{n}}[A^\star]^\top(A A^\top - \mathrm{I}_k) A y  \\
        &= 0,
    \end{align}
    where we have used the fact that $\mathbb{E}[\phi_i^\star(x_i^{(j)})[\phi_i^\star(x_i^{(j)})]^\top] = 1$, $Q^2 = \mathrm{I}_k$ and $(A A^\top - \mathrm{I}_k) A = 0$.
    The rest of the proof is concluded by using the Bernstein inequality by following directly the steps detailed in \citet[Proof of Lemma 3, see equations (35) to (39)]{pmlr-v139-collins21a}. 
\end{proof}

\begin{lemma}
    Assume \Cref{ass}.
    Let $\delta_d = c d^{3/2}\sqrt{\log(\floor{rb})} / n^{\half}$ for some absolute constant $c > 0$.
    Then, for any $t \in [T]$, with probability at least $1 - e^{-111k^2\log(\floor{rb})}$, we have
    $$ 
    \norm{F^{(t)}}_F\leq \frac{\delta_d}{1-\delta_d} \ \norm{B^\star_{\mathsf{A}_t}}_2 \ \mathrm{dist}\pr{A^{(t)},Q A^\star}\eqsp,
    $$
    where $F^{(t)}$ is defined in \eqref{eq:F} and $A^{(t)}$ in \eqref{eq:A}. 
\end{lemma}
\begin{proof}
    By the Cauchy-Schwarz inequality, we have $\norm{F^{(t)}}_F = \norm{[G^{(t)}]^{-1}(G^{(t)}D^{(t)} - C^{(t)})B^\star_{\mathsf{A}_t}}_2\leq \delta_d \ \norm{B^\star_{\mathsf{A}_t}}_2 \leq \norm{[G^{(t)}]^{-1}}_2 \norm{(G^{(t)}D^{(t)} - C^{(t)})B^\star_{\mathsf{A}_t}}_2\leq \delta_d \ \norm{B^\star_{\mathsf{A}_t}}_2$. The proof is concluded by combining the upper bounds given in \Cref{lemma:G} and \Cref{lemma:GB}.
\end{proof}

\begin{lemma}
    \label{lemma:AA}
    Assume \Cref{ass} and let $\delta_d' = c d \sqrt{k} / \sqrt{\floor{rb}n}$ for some absolute positive constant $c$. 
    For any $t \in [T]$ and whenever $\delta_d' \leq d$, we have with probability at least $1 - e^{-110k} - e^{-110d^2 \log(\floor{rb})}$
    $$
    \frac{1}{\floor{rb}}\norm{\pr{\frac{1}{n}Q (\mathcal{A}^{(t)})^\dagger\mathcal{A}^{(t)}\pr{Z^{(t)}}Q - Z^{(t)}}^\top B^{(t)}_{\mathsf{A}_t}}_2 \leq \delta_d' \ \mathrm{dist}\pr{A^{(t)},QA^\star} \eqsp,
    $$
    where $B_{\mathsf{A}_t}^{(t)}$ is defined in \eqref{eq:B} and $Z^{(t)}$ is defined in \eqref{eq:Z}.
\end{lemma}

\begin{proof}
    Let $t \in [T]$.
    Note that we have 
    \begin{align}
        \pr{\frac{1}{n}Q (\mathcal{A}^{(t)})^\dagger\mathcal{A}^{(t)}\pr{Z^{(t)}}Q - Z^{(t)}}^\top B^{(t)}_{\mathsf{A}_t} 
        &= \frac{1}{n} \sum_{i \in \mathsf{A}_t} \sum_{j=1}^m \langle Q\phi^\star_i(x_i^{(j)}),z_i^{(t)}\rangle Q\phi^\star_i(x_i^{(j)})\br{\beta_i^{(t)}}^\top - z_i^{(t)}\br{\beta_i^{(t)}}^\top\eqsp.
    \end{align}
    Let $\mathcal{S}^{k-1}$ and $\mathcal{S}^{d-1}$ denote the $k$-dimensional and $d$-dimensional unit spheres, respectively. 
    Then, by \citet[Corollary 4.2.13]{vershyninHighdimensionalProbabilityIntroduction2018}, we can define $\mathcal{N}_k$ and $\mathcal{N}_d$, $1/4$-nets over $\mathcal{S}^{k-1}$ and $\mathcal{S}^{d-1}$, respectively, such that $|\mathcal{N}_k|\leq 9^k$ and $|\mathcal{N}_d|\leq 9^d$.
    Therefore, by using \citet[Equation (4.13)]{vershyninHighdimensionalProbabilityIntroduction2018}, we have
    \begin{align}
        &\norm{\pr{\frac{1}{n}Q (\mathcal{A}^{(t)})^\dagger\mathcal{A}^{(t)}\pr{Z^{(t)}}Q - Z^{(t)}}^\top B^{(t)}_{\mathsf{A}_t}}_2^2 \\ 
        &= 2 \max_{u \in \mathcal{N}_d, v \in \mathcal{N}_k} u^\top \br{\frac{1}{n} \sum_{i \in \mathsf{A}_t} \sum_{j=1}^m \langle Q\phi^\star_i(x_i^{(j)}),z_i^{(t)}\rangle Q\phi^\star_i(x_i^{(j)})\br{\beta_i^{(t)}}^\top - z_i^{(t)}\br{\beta_i^{(t)}}^\top} v \\
        &= 2 \max_{u \in \mathcal{N}_d, v \in \mathcal{N}_k} \frac{1}{n} \sum_{i \in \mathsf{A}_t} \sum_{j=1}^m \langle Q\phi^\star_i(x_i^{(j)}),z_i^{(t)}\rangle \langle u, Q\phi^\star_i(x_i^{(j)})\rangle \langle \beta_i^{(t)},v \rangle - \langle u,z_i^{(t)} \rangle \langle \beta_i^{(t)}, v \rangle\eqsp.\label{eq:control_rec}
    \end{align}
    In order to control \eqref{eq:control_rec} using Bernstein inequality as in \Cref{lemma:GB}, we need to characterise, in particular, the sub-Gaussianity of $\langle u,z_i^{(t)}\rangle$ and $\langle \beta_i^{(t)},v\rangle$ which require a bound on $\|z_i^{(t)}\|$ and $\|\beta_i^{(t)}\|$, respectively.
    From \Cref{lemma:lemma_recursionB}, we have $[\beta_i^{(t)}]^\top = (\beta_i^\star)^\top (A^\star)^\top A^{(t)} - (z_i^{(t)})^\top$ which leads to
    \begin{align}
        \norm{z_i^{(t)}}^2 
        &=  \norm{QA^{(t)}(A^{(t)})^\top Q A^\star \beta_i^\star - Q A^{(t)} f_i^{(t)} - A^\star \beta_i^\star}_2^2 \\
        &= \norm{(QA^{(t)}(A^{(t)})^\top Q - \mathrm{I}_d) A^\star \beta_i^\star - Q A^{(t)} f_i^{(t)}}_2^2 \\
        &\leq 2\norm{(QA^{(t)}(A^{(t)})^\top Q - \mathrm{I}_d) A^\star}_2^2 \norm{\beta_i^\star}^2 + 2 \norm{f_i^{(t)}}^2 \\
        &\leq 2 d \ \mathrm{dist}^2(A^{(t)},QA^\star) + 2\norm{f_i^{(t)}}^2\eqsp.
    \end{align}
    Using \eqref{eq:F} and the Cauchy-Schwarz inequality, we have 
    \begin{align}
        \norm{f_i^{(t)}}^2 
        &= \norm{[G^{i,(t)}]^{-1}(G^{i,(t)}D^{i,(t)} - C^{i,(t)})\beta^\star_i}^2\\
        &\leq \norm{[G^{i,(t)}]^{-1}}^2_2 \norm{G^{i,(t)}D^{i,(t)} - C^{i,(t)}}^2_2 \norm{\beta^\star_i}^2\\
        &\leq d \norm{[G^{i,(t)}]^{-1}}^2_2 \norm{G^{i,(t)}D^{i,(t)} - C^{i,(t)}}^2_2\eqsp,\label{eq:fi}
    \end{align}
    where the last inequality follows from \Cref{ass}-\ref{ass:2}.

    Using \Cref{lemma:G} and \Cref{lemma:GB} and similarly to \citet[Equation (45)]{pmlr-v139-collins21a}, it follows for any $i \in \mathsf{A}_t$ that 
    $$
    \norm{z_i^{(t)}}^2_2 \leq 4 d \ \mathrm{dist}(A^{(t)},QA^\star)\eqsp,
    $$ 
    with probability at least $1 - e^{110d^2\log(\floor{rb})}$.

    Similarly, using \Cref{lemma:lemma_recursionB} and \eqref{eq:fi}, we have with probability at least $1 - e^{110d^2\log(\floor{rb})}$ and for any $i \in \mathsf{A}_t$, that 
    \begin{align}
        \norm{\beta_i^{(t)}}^2 \leq 2 \norm{[A^{(t)}]^\top Q A^\star \beta_i^\star}^2 + 2\norm{f_i^{(t)}}^2 \leq 4d \eqsp.
    \end{align}

    Besides, note we have 
    \begin{align}
        \mathbb{E}\br{\langle Q\phi^\star_i(x_i^{(j)}),z_i^{(t)}\rangle \langle u, Q\phi^\star_i(x_i^{(j)})\rangle \langle \beta_i^{(t)},v \rangle} 
        &= \langle u,z_i^{(t)}\rangle \langle \beta_i^{(t)},v \rangle \eqsp.
    \end{align}

    The proof is then concluded by applying the Bernstein inequality following the same steps as in the final steps of \citet[Proof of Lemma 5]{pmlr-v139-collins21a}.
\end{proof}

We are now ready to control $C_2$.

\begin{lemma}
    \label{lemma:C2}
    Assume \Cref{ass} and let $\delta_d' = c d \sqrt{k} / \sqrt{\floor{rb}n}$ for some absolute positive constant $c$. 
    For any $t \in \{0,\ldots,T-1\}$, $\eta >0$ and whenever $\delta_d' \leq d$, we have with probability at least $1 - e^{-110k} - e^{-110d^2 \log(\floor{rb})}$
    $$
    C_2 \leq \eta \delta'_d \ \mathrm{dist}\pr{A^{(t)},QA^\star}\norm{\pr{R^{(t+1)}}^{-1}}_2 \eqsp,
    $$
    where $C_2$ is defined in \eqref{eq:C2}, $A^{(t)}$ is defined in \eqref{eq:A} and $R^{(t)}$ comes from the QR factorisation of $\bar{A}^{(t)}$, see step 20 in \Cref{algo:FLIC-FedRep-linreg}.
\end{lemma}
\begin{proof}
Let $t \in \{0,\ldots,T-1\}$ and $\eta > 0$.
Then, whenever $\delta_d' \leq d$, we have with probability at least $1 - e^{-110k} - e^{-110d^2 \log(\floor{rb})}$, we have
\begin{align}
     C_2 
     &= \frac{\eta}{\floor{rb}} \norm{\pr{\frac{1}{n}[A^\star_\perp]^\top (Q\mathcal{A}^{(t+1)})^\dagger\mathcal{A}^{(t+1)}\pr{Z^{(t+1)}}Q - Z^{(t+1)}}^\top B^{(t+1)}_{\mathsf{A}_{t+1}}}_2 \norm{\pr{R^{(t+1)}}^{-1}}_2\\
     &\leq \frac{\eta}{\floor{rb}} \norm{\pr{\frac{1}{n} (Q\mathcal{A}^{(t+1)})^\dagger\mathcal{A}^{(t+1)}\pr{Z^{(t+1)}}Q - Z^{(t+1)}}^\top B^{(t+1)}_{\mathsf{A}_{t+1}}}_2 \norm{\pr{R^{(t+1)}}^{-1}}_2 \\
     &\leq \eta \delta'_d \ \mathrm{dist}\pr{A^{(t)},QA^\star}\norm{\pr{R^{(t+1)}}^{-1}}_2\eqsp,
\end{align}
where we used the Cauchy-Schwarz inequality in the second inequality and \Cref{lemma:AA} for the last one.
\end{proof}

\noindent \textbf{Control of $\|\pr{R^{(t+1)}}^{-1}\|_2$.}
To finalise our proof, it remains to bound $\|\pr{R^{(t+1)}}^{-1}\|_2$. The associated result is depicted in the next lemma.

\begin{lemma}
    \label{lemma:R}
    Define $\bar{\delta}_d = \delta_d + \delta_d'$ where $\delta_d$ and $\delta_d'$ are defined in \Cref{lemma:G} and \Cref{lemma:GB}, respectively.
    Assume \Cref{ass}. Then, we have with probability at least $1 - e^{-110k} - e^{-110d^2 \log(\floor{rb})}$,
    $$
    \norm{\pr{R^{(t+1)}}^{-1}}_2 \leq \pr{1 - 4 \eta \frac{\bar{\delta}_d}{(1-\bar{\delta}_d)^2}\bar{\sigma}_{\max,\star}^2}^{-\half}\eqsp.
    $$
\end{lemma}
\begin{proof}
    The proof follows from \citet[Proof of Lemma 6]{pmlr-v139-collins21a}.
\end{proof}

\section{Experimental Details}
\label{sec:expe}

%
%

\subsection{Reference Distribution for Regression}

For regression problem, our goal is to map all samples for all clients into a 
common latent subspace, in which some structural information about regression problem is preserved.
As such, in order to reproduce the idea of using a Gaussian mixture model as a anchor
distribution, we propose to use an infinite number of Gaussian mixtures in which the distribution of $x$ associated to a response $y$  is going to  be mapped on a unit-variance Gaussian distribution whose mean depends uniquely on $y$. 
Formally, we define the anchor distribution as 
$$
\mu_y = \gauss(\mathrm{m}^{(y)},\mathrm{I})
$$
where $\mathrm{m}^{(y)}$ is a vector of dimension $d$ that is uniquely defined. In practice, we consider  as $\mathrm{m}^{(y)} = y a + (1-y) b$ where $a$ and $b$ are two vectors in $\mathbb{R}^d$.

When training \ours, this means that for a client $i$, we can compute $\mathrm{W}_2^2 \pr{\mu_y,\nu_{\phi_i}^{(y)}}$ based on the set of  training samples $\{x,y\}$. In practice, if for a given batch of samples we have a single sample  $x$, then the Wasserstein distance boils
to $\|\phi_i(x) - \mathrm{m}^{(y)}\|_2^2$.

\subsection{Data Sets}
We provide some details about the datasets we used for our numerical experiments

\subsubsection{Toy data sets}
The first toy dataset, denoted as \emph{noisy features}, is a $20$-class classification problem in which the features for a given class is obtained by sampling a Gaussian distribution of dimension $5$, with random mean and Identity covariance matrix. 
For building the training set, we sample $2000$ examples for each class and equally share
those examples among clients who  hold that class. Then, in order to generate some class imbalances on clients, we randomly subsample examples on all clients.
For instance, with $100$ clients and 2 classes per clients, this results in a problem
with a total of about $16$k samples with a minimal number of samples of $38$ and a maximal one of $400$. In order to get different dimensionality, we randomly append on each client dataset some Gaussian random noisy features with dimensionality varying from $1$ to $10$.

The second toy dataset, denoted as \emph{linear mapping}, is a $20$-class classification problem where each class-conditional distribution is
Gaussian distribution of dimension $5$, with random mean and random diagonal covariance
matrix. As above, we generate $2000$ samples per class and distribute and subsample them across clients in the similar way, leading to a total number of samples of about $15k$.  The dimensionality perturbation is modelled
by a random (Gaussian)linear transformation that maps the original samples to a space
which dimension goes up to $50$.

\subsubsection{MNIST-USPS}
We consider a digit classification problem with the original MNIST and USPS data sets which are respectively of dimension $28 \times 28$ and $16 \times 16$ and we assume that a client hosts either a subset of MNIST or USPS data set. We use the natural train/test split of those datasets and randomly share them accross clients. 

\subsubsection{TextCaps data set}

The TextCaps data set \citep{sidorov2020textcaps} is an Image captioning dataset for which goal is
to develop a model able to produce a text that captions the image.
The dataset is composed of about 21k images and 110k captions and
each image also comes with an object class. For our purpose, we have extracted
pair of $14977$ images and captions from the following four classes \textit{Bottle}, \textit{Car}, \textit{Food} and \textit{Book}. At each run, those pairs are separated in $80\%$ train and  $20\%$ test sets. Examples from the TextCaps datasets are presented
in \Cref{fig:textcaps}.
 Images and captions are
represented by vectors by feeding them respectively to a pre-trained ResNet18 
and a pretrained Bert, leading to vectors of size $512$ and $768$. 

Each client holds  either the image or the text representation of subset of
examples and the associated  vectors are randomly pruned of up to $10\%$ coordinates.
As such, all clients hold dataset with different dimensionality.

\subsection{Models and Learning Parameters}
\label{append:models}
For the toy problems and the \emph{TextCaps data set}, as a local transformation functions we used a fully connected neural network with one input, one hidden layer   and one output layers. The number
of units in hidden layer has been fixed to $64$ and the dimension of latent space as
been fixed to $64$.  ReLU activation has been applied after the input and hidden layers. For the digits dataset, we used
a CNN model with 2 convolutional layers  followed by a max-pooling layer and a sigmoid activation function. Once flattened, we have a
one fully-connected layer and {ReLU activation}. The latent dimension is fixed to $64$. 

For all datasets, as for the local model $g_{\theta_i}$, in order to be consistent with competitors, we first considered a single layer linear model implementing the local classifier as well as a model with one input layer (linear units followed by a LeakyReLU activation funcion) denoting the shared representation layer and an output linear layer.

For training, all methods use Adam with a default learning rate of $0.001$ and a batch size of $100$. Other hyperparameters have been set as follows. Unless specified, the regularization strength
$\lambda_1$ and $\lambda_2$ have been fixed to $0.001$. Local sample batch size is set to $100$ and the participation rate $r$ to $0.1$. For all experiments, we have set the number of communication round $T$ to $50$ and the number of local epochs to respectively $10$ and $100$ for the real-world and toy datasets. For \ours, as in FedRep those local epochs is followed by
one epoch for representation learning. We have trained the local embedding functions for $100$ local epochs and a batch size of $10$  for toy datasets and TextCaps and while of $100$ for  MNIST-USPS. 
Reported accuracies are computed
after local training for all clients.

\subsection{Ablating Loss Curves}

\begin{figure*}
    \begin{center}
        \includegraphics[width=5.5cm]{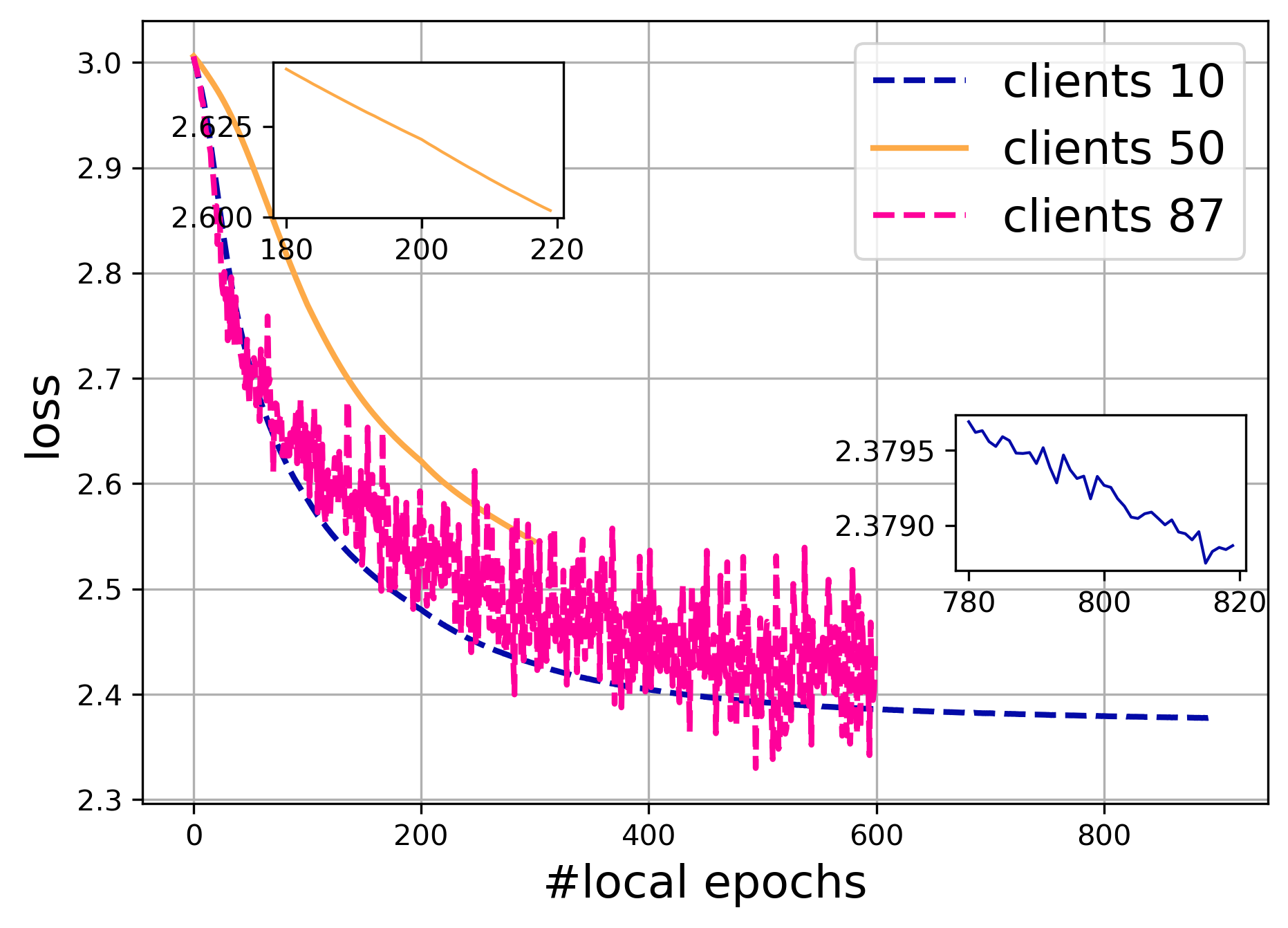}
        \includegraphics[width=5.5cm]{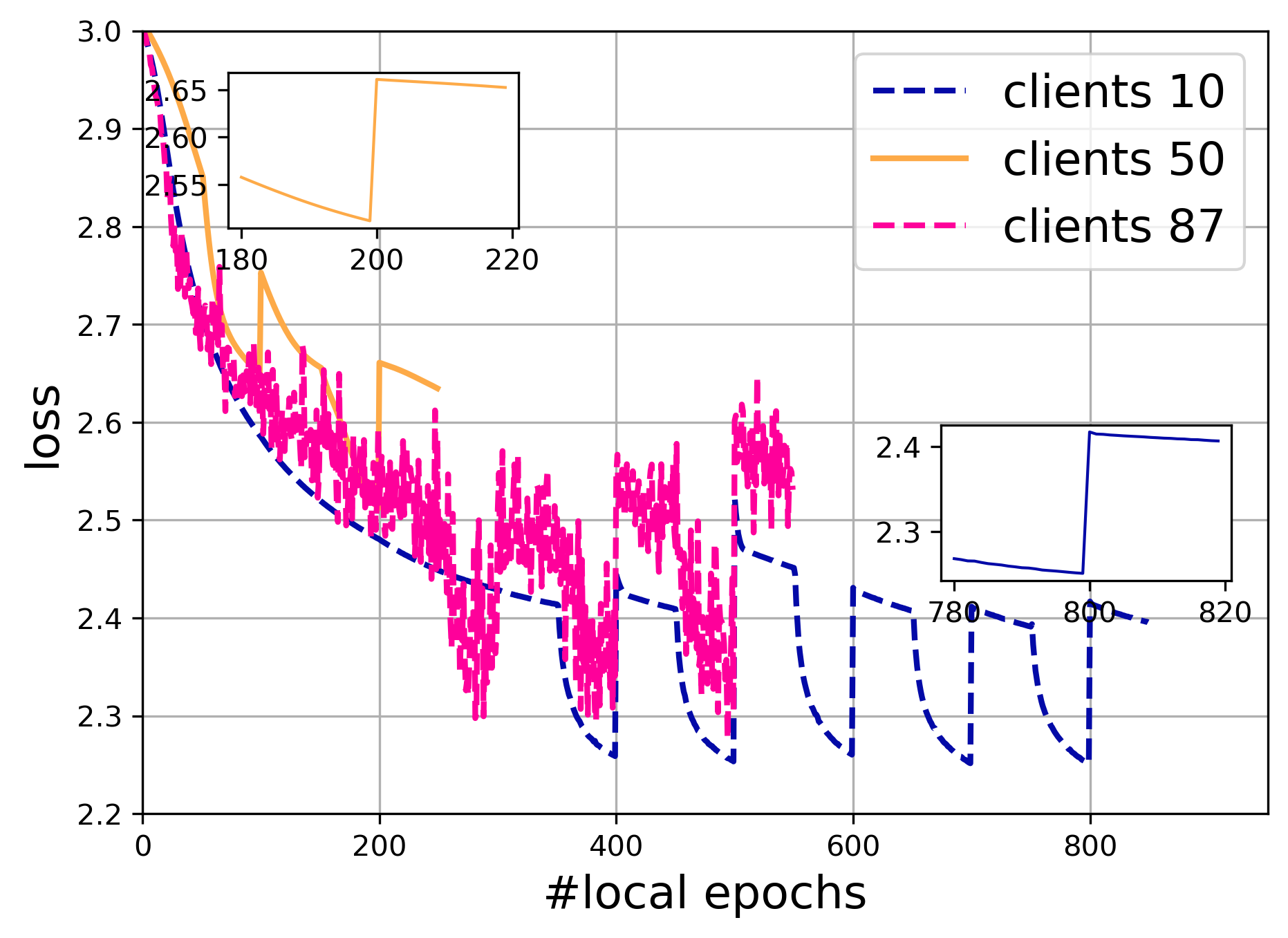}
        \includegraphics[width=5.5cm]{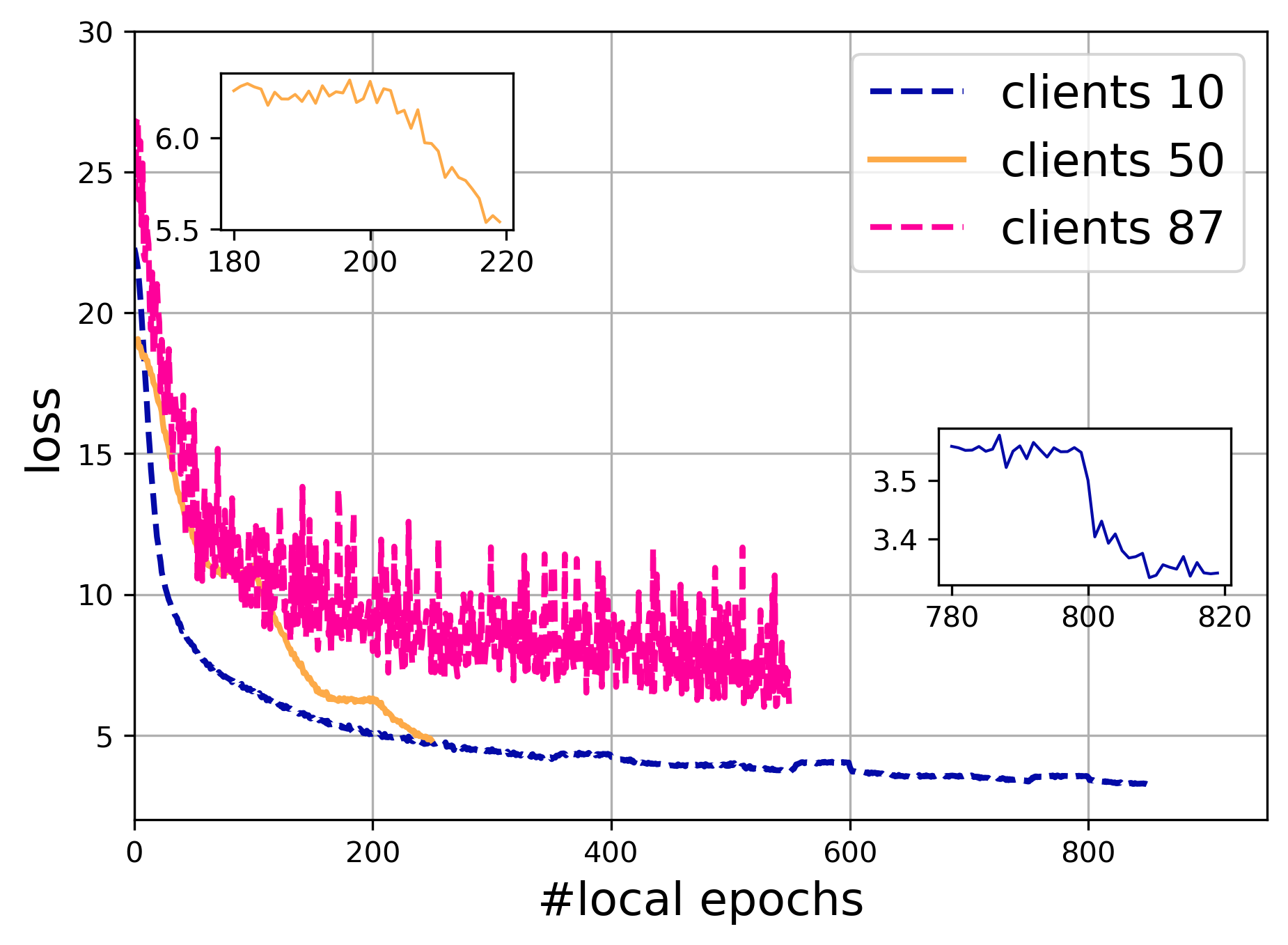}
    \end{center}
    \caption{Evolution of the local loss curve of three different clients for three different learning situations. See text for details. \label{fig:losscurve}}
\end{figure*}

In order to gain some understanding on the learning mechanism that involves local and global training respectively due to the local embedding functions, the local classifier and the global representation learning, we propose to look at local loss curves across different clients.

Here, we have considered the \emph{linear mapping} toy dataset as those used
in the toy problem analysis. However, the learning parameters we have chosen are different from those we have used to produce the results so as to highlight some specific features. 
The number of epochs (communication rounds) is set to $100$ with a client activation
ration of $0.1$. Those local epochs are shared for either training the local parameters or the global ones (note that in our reference \Cref{algo:FLIC-FEDREP}, the global parameter is updated only once for each client)
Those latter are trained starting after the   $20$-th communication round and 
in this case, the local epochs are equally shared between local and global parameter
updates. Note that
because of the randomness in the client selection at each epoch, the total number of local epochs is different from client to client.   We have evaluated three learning situations and plotted the 
loss curves for each client.
\begin{itemize}
    \item the local embedding functions and the global models are kept fixed, and only the local classifier is trained. Examples of loss curves for $3$ clients are presented in the left plot of \Cref{fig:losscurve}. For this learning situation,    there is no shared global parameters that are trained locally. Hence, the loss curve is typical of those  obtained by stochastic gradient descent with a smooth transition, at multiple of $100$ local epochs, when  a given client is retrained after a communication rounds.
    \item the local embedding functions are kept fixed, while the classifier and global parameters are updated using half of the local epochs each. This situation is interesting and reported in middle plot in \Cref{fig:losscurve}. We can see that for 
    some rounds of $100$ local epochs, a strong drop in the loss occurs at
    starting at the $50$th local epoch because the global parameters are being updated. Once the local update of a client is finished the global parameter is sent back to the server and all updates of global parameters are averaged by the server. When
    a client is selected again for local updates, it is served with a new global parameter (hence a new loss value ) which causes the discontinuity in the loss curve at the beggining of each local update.

    \item all the part (local embedding functions, global parameter and the classifier) of the models are trained. Note at first that the loss value for those curves (right plot in \Cref{fig:losscurve}) is larger than for the two first most left plots as the Wasserstein distance to the anchor distribution  is now taken into account and tends to dominate the loss.  The loss curves are globally decreasing with larger drops in loss at the beginning of local epochs.
    
\end{itemize}

\subsection{On Importance of Alignment Pre-Training and Updates.} 
\label{sec:alignment}

We have analyzed the impact of pretraining the local transformation functions and their updates during learning for fixed reference distribution. We have   considered two learning situations : one in which they are updated during local training (as usual) and another one in they are kept fixed all along the training. 
We have chosen the  setting with $100$ users and have kept the same experimental settings as for the performance figure and made only varied the number of epochs considered for pretraining from
$1$ to $200$. 
\begin{figure}[t]
    \begin{center}
        \includegraphics[width=7cm]{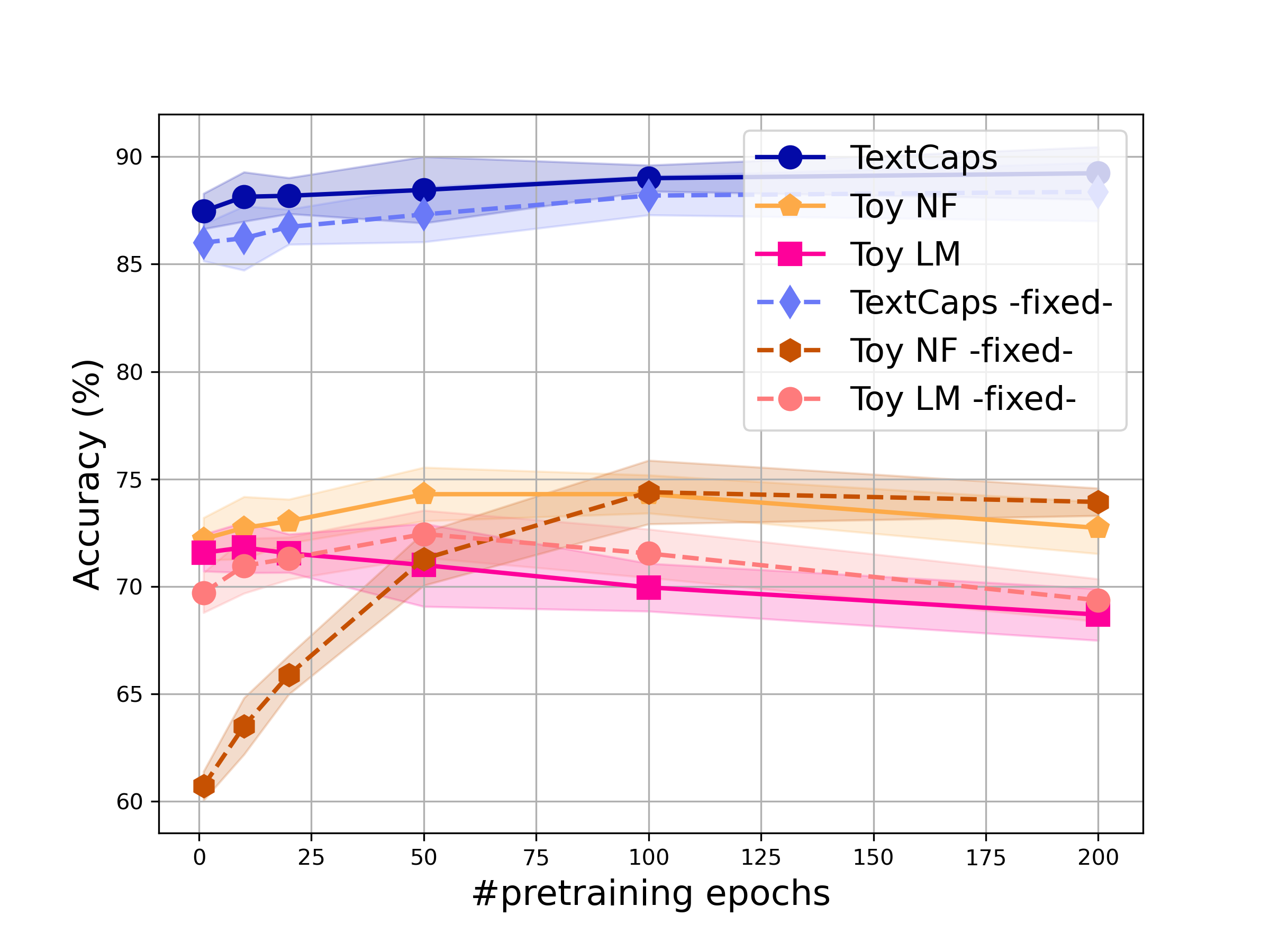}
        \caption{Impact of epochs used for pretraining $\phi_i$ on the model accuracy as well as updating those functions during the training. Results for three different datasets are reported. Plain and dashed curves are respectively related to local training with
                and without updates $\phi_i$. 
                \label{fig:pretraining}}
    \end{center}
    
    \end{figure}
    Results, averaged over $5$ runs are shown in \Cref{fig:pretraining}.
    We remark that for the three datasets, increasing the number of epochs up to
    a certain number tends to increase performance, but overfitting may occur.
    The latter is mostly reflected in the \emph{toy linear mapping} dataset for which $10$ to $50$
    epochs is sufficient for good pretraining.  Examples of how classes evolves
    during pretraining are illustrated in \Cref{fig:tsne-toy}, through \emph{t-sne} projection. We also illustrate cases of how pretraining impact on the test set and may lead to overfitting as shown in the supplementary \Cref{fig:tsne-toybig}.

\subsection{On the Impact of the Participation Rate}
\label{sec:participation}
We have analyzed the effect of the participation rate of each client into our federated learning approach. \Cref{fig:participation} reports the accuracies, averaged over $3$ runs, of our approach for the toy datasets and the \emph{TextCaps} problem with respect
to the partication rate at each round. We can note that the proposed approach is rather
robust to the participation rate but may rather suffer from overfitting due
to overtraining of local models. On the left plot, performances, measured
after the last communication round, for \emph{TextCaps} is stable over participation
rate while those performances tend to decrease for the \emph{toy} problems. We
associate these decrease to overfitting since  when we report (see right plot)
the best performance over communication rounds (and not the last one), they are stable
for all problems. This suggests that number of local epochs may be dependent
to the task on each client and the  client participation rate.

\begin{figure*}
    \begin{center}
        ~\hfill\includegraphics[width=7.5cm]{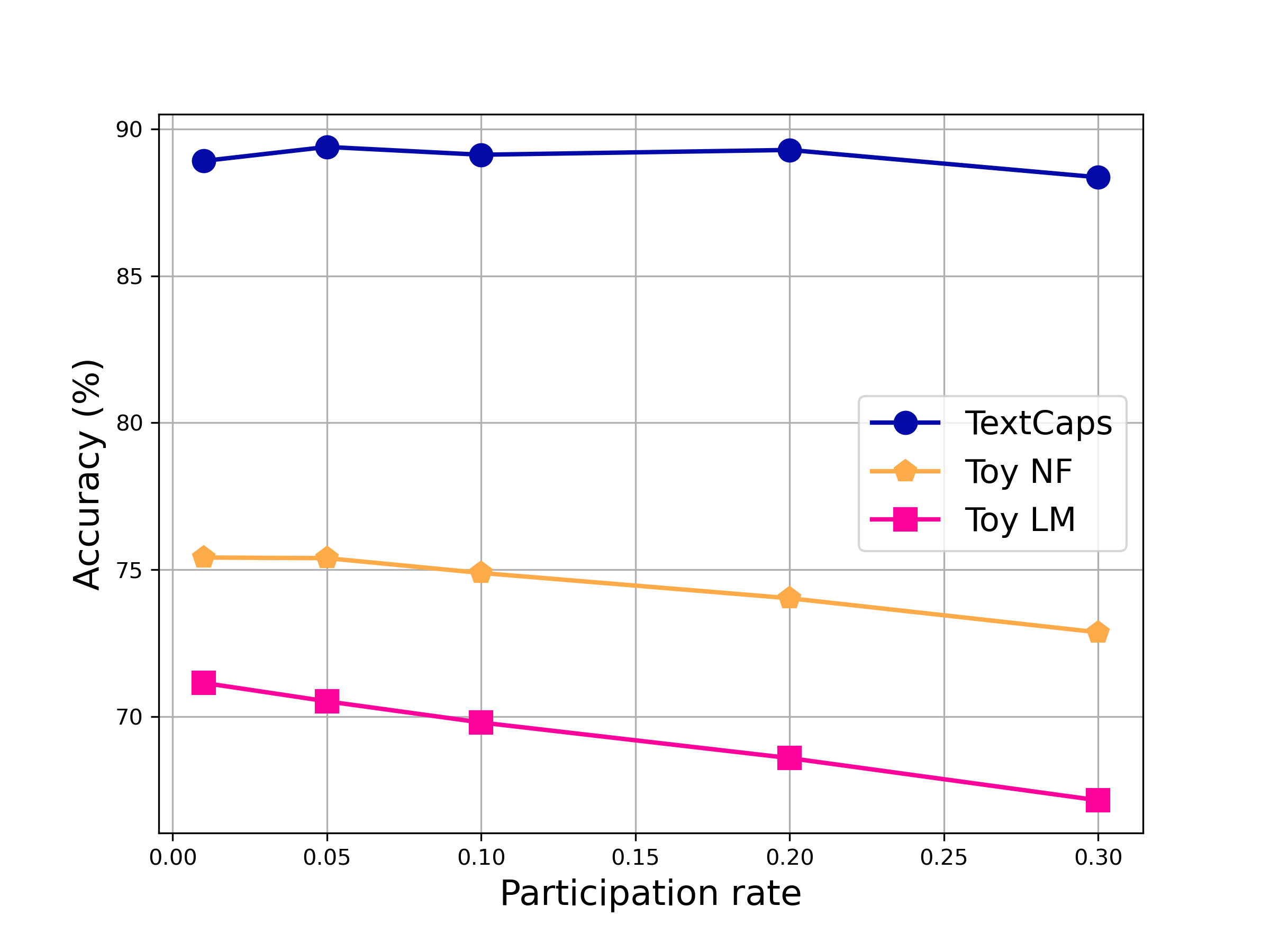}~\hfill~
        \includegraphics[width=7.5cm]{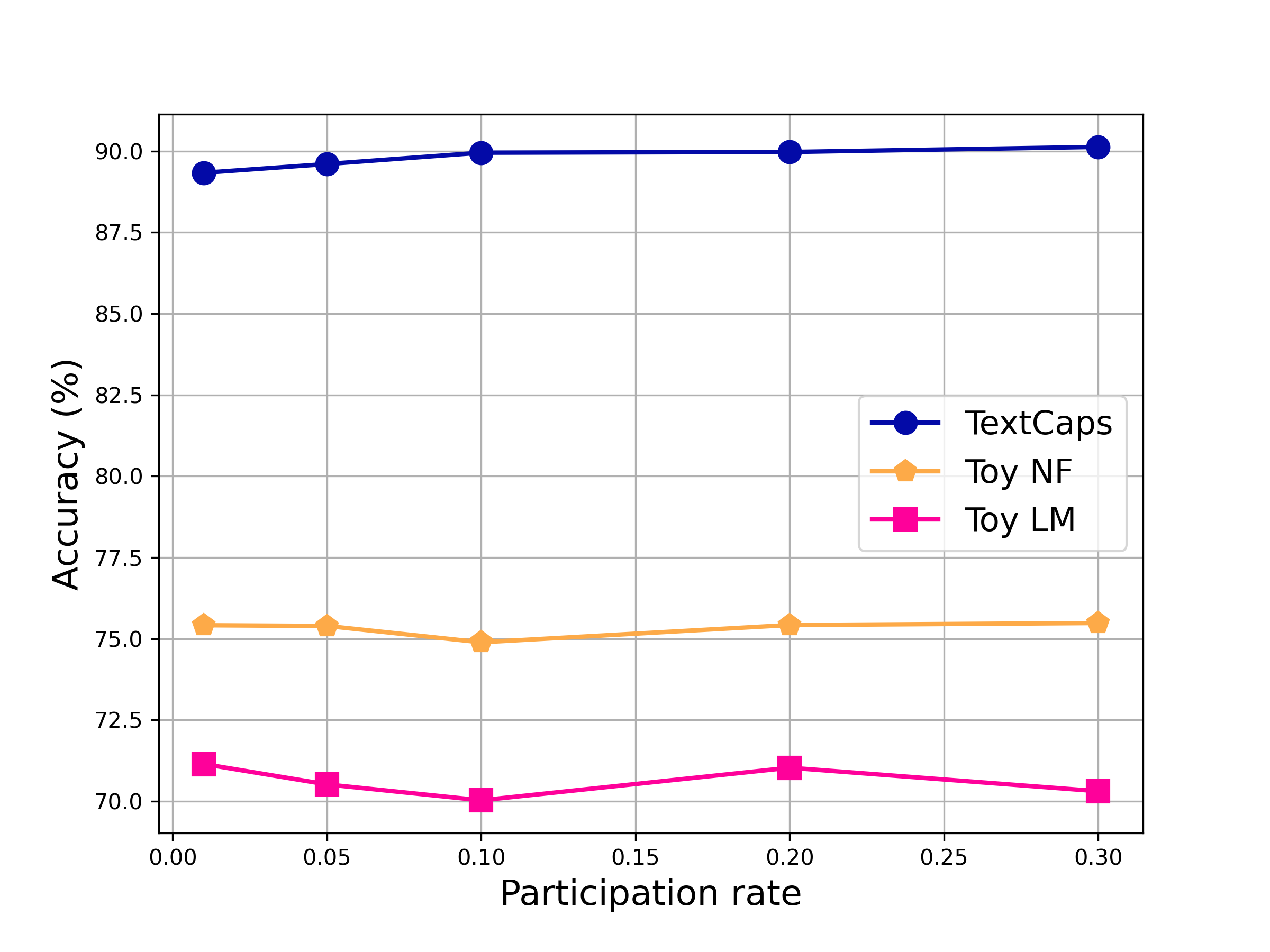}\hfill~
    \end{center}
\caption{Evolution of the performance of our \ours-Class algorithm with respects to
the participation rate of clients, using the same experimental setting as in
\Cref{fig:toyperf}. (left) evaluating performance after last communication rounds, (right)
best performance across communication rounds. \label{fig:participation} }
\end{figure*}

\begin{figure*}
    \begin{center}
    \includegraphics[width=5cm]{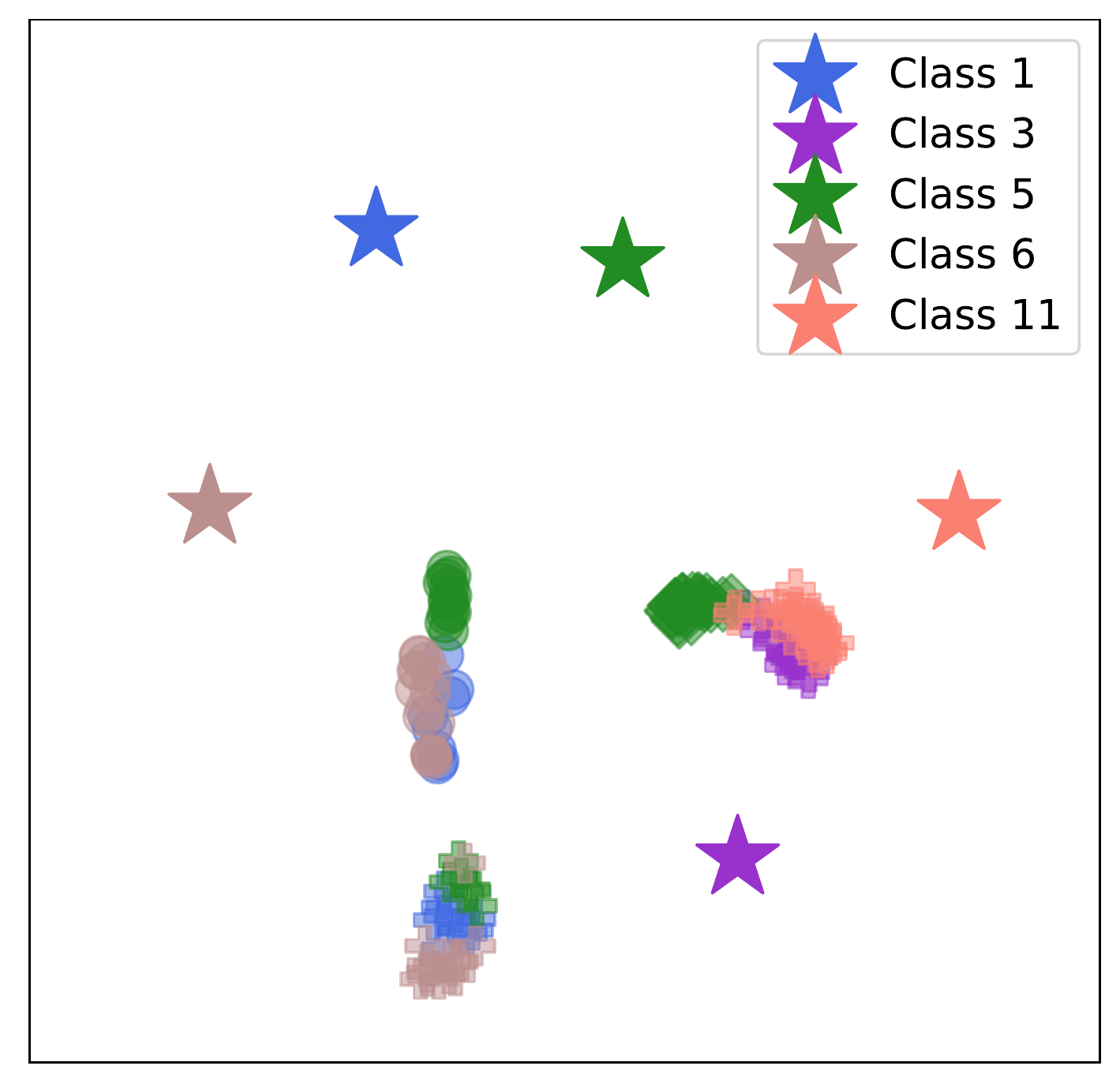}
    \includegraphics[width=5cm]{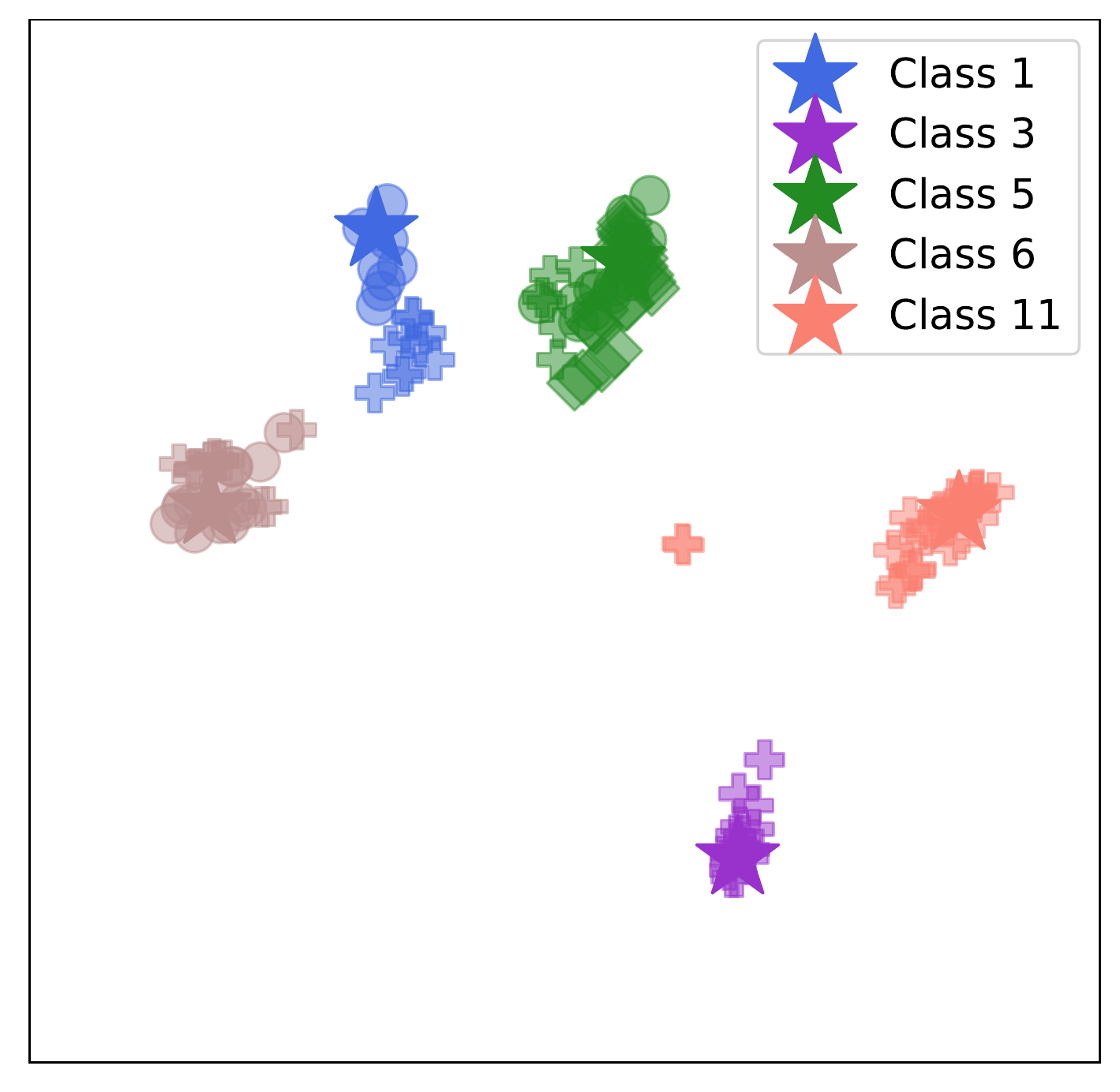}
    \includegraphics[width=5cm]{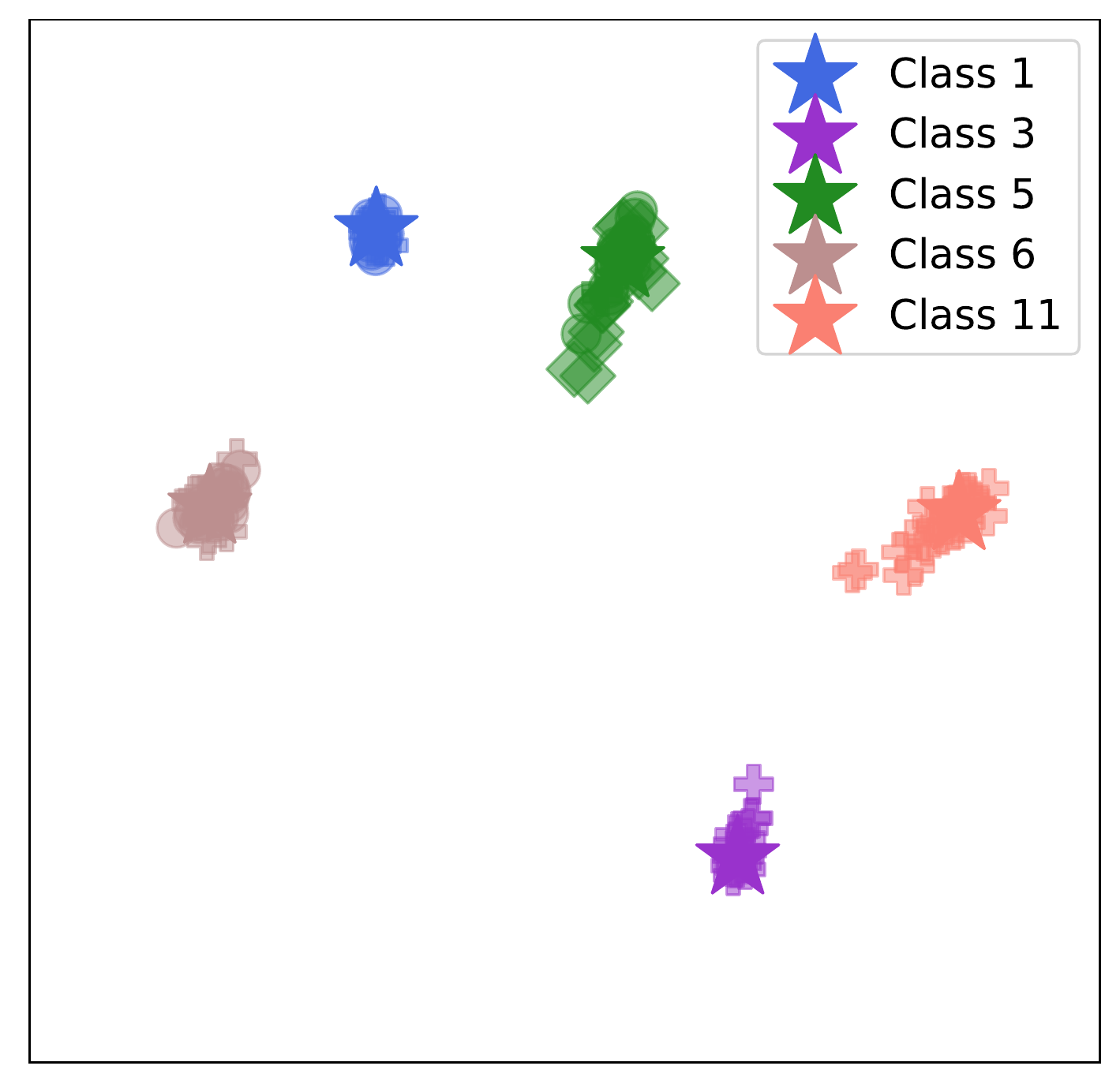} \\
    \includegraphics[width=5cm]{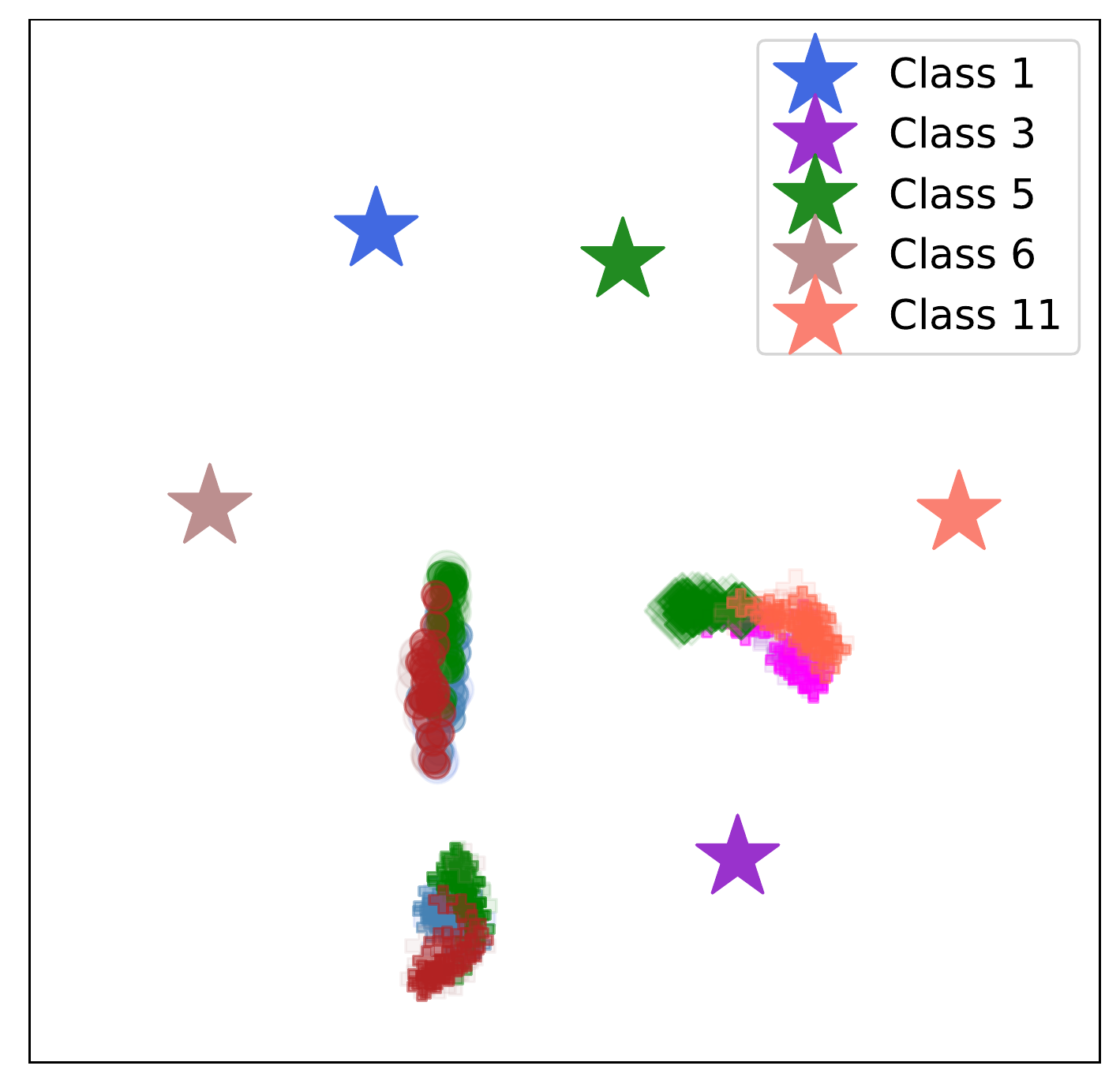}
    \includegraphics[width=5cm]{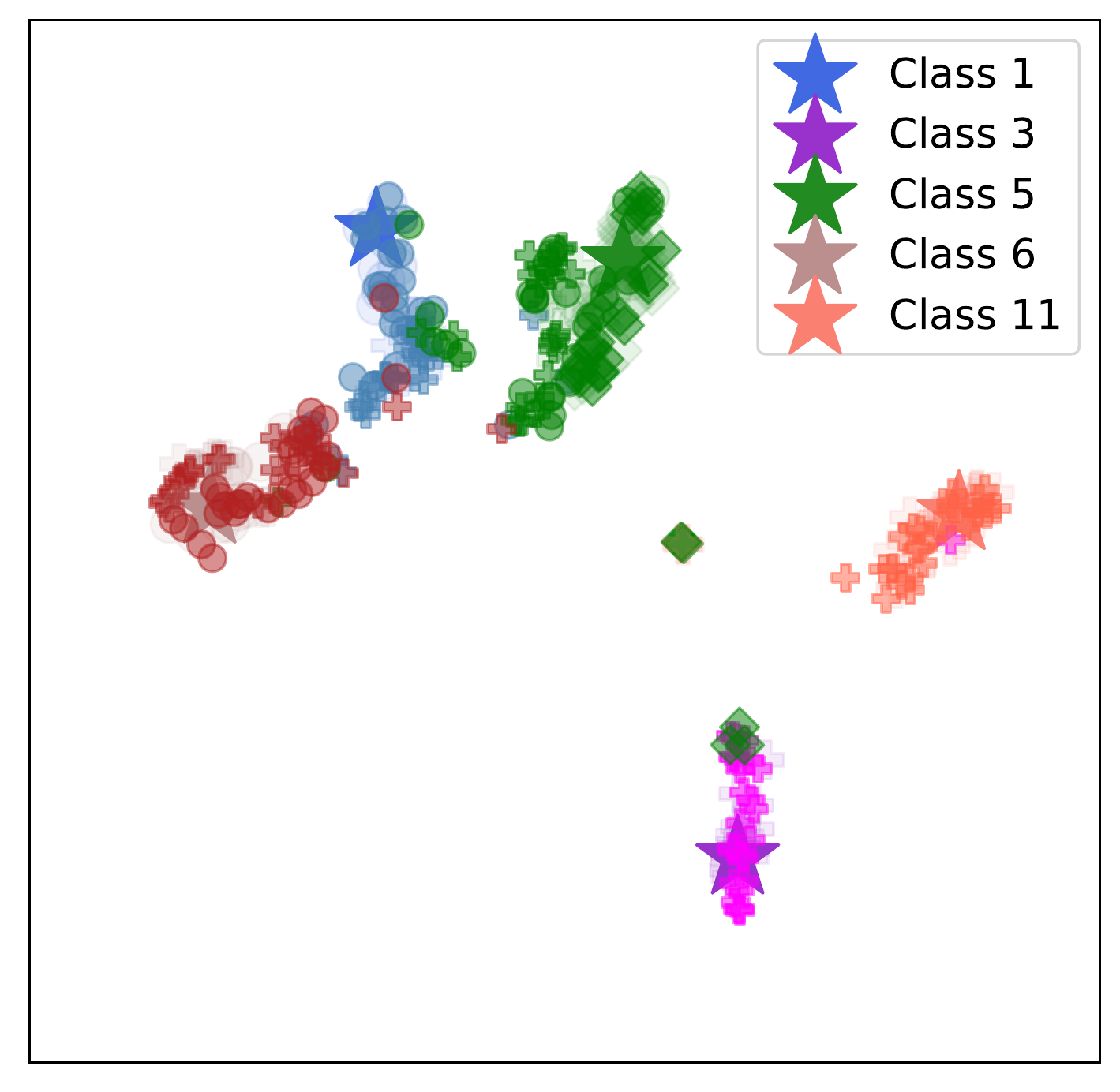}
    \includegraphics[width=5cm]{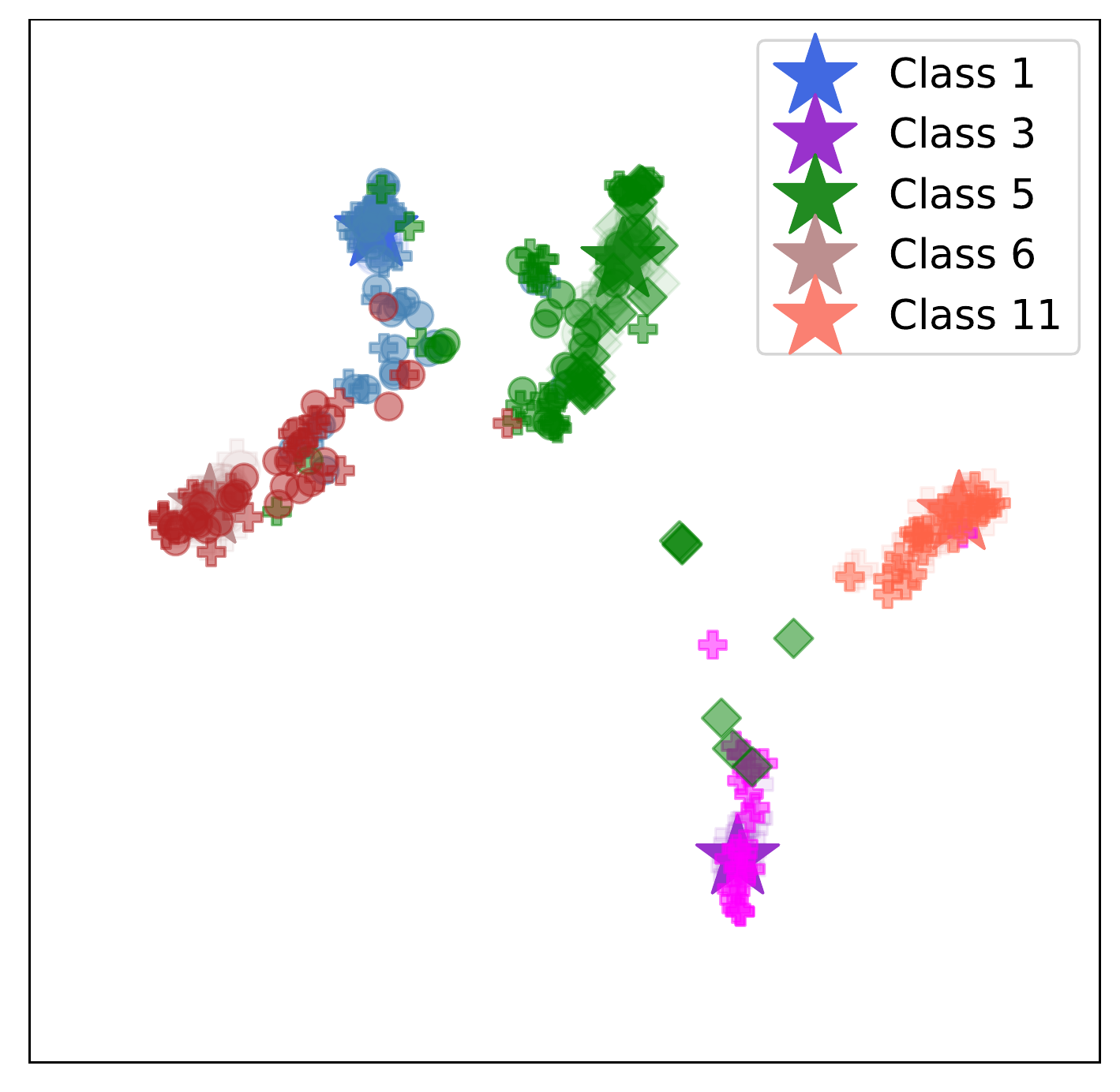} \\
\end{center}

    \caption{\label{fig:tsne-toybig}. 2D \emph{t-sne} projection of $5$ classes partially shared by $3$ clients for the \textbf{toy linear mapping} dataset after learning the local embedding functions for (left) 10 epochs, (middle) 50 epochs, (right) 100 epochs. Original dimensions on clients vary from $5$ to $50$.
    Top row shows the projection the training set while bottom row plots show both training and test set. Star $\star$ markers represent the projection of the
    mean of each class-conditional.
     The three different marker styles represent the different clients. Classes are
     denoted by colors and similar tones of color distinguish train and test sets.
     We  see that each class from the training set from each client converges towards the mean of its anchor distribution, represented by the star marker. Interestingly, we also remark that unless convergence is reached, empirical class-conditional distributions on each clients are not equal making necessary the learning of a joint representation. From the bottom plots, we
     can understand that distribution alignment impacts mostly the training set 
     but this alignment does not always generalize properly to the test sets.
}
\end{figure*}

\newpage
\begin{figure*}
	~\hfill\includegraphics[width=0.5\linewidth]{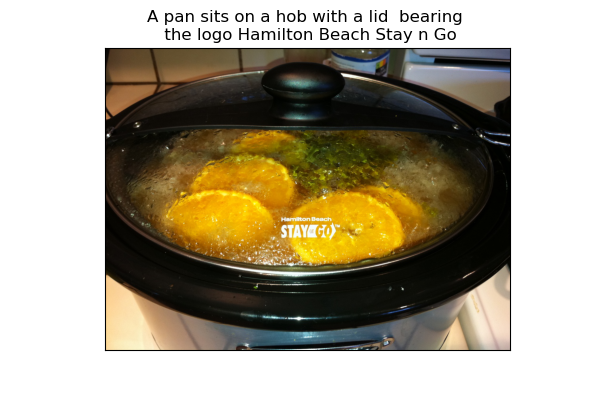}~\hfill~
		~\hfill\includegraphics[width=0.5\linewidth]{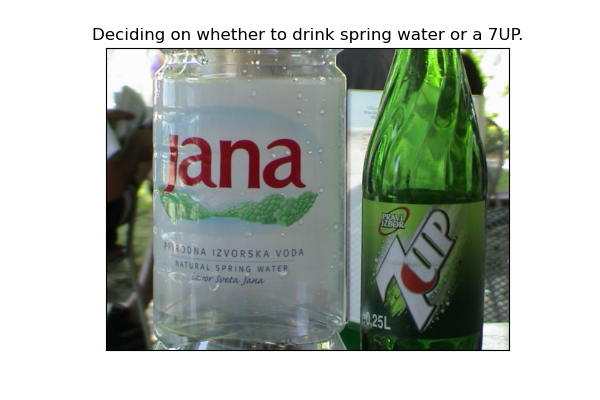}~\hfill~ \\
			~\hfill\includegraphics[width=0.5\linewidth]{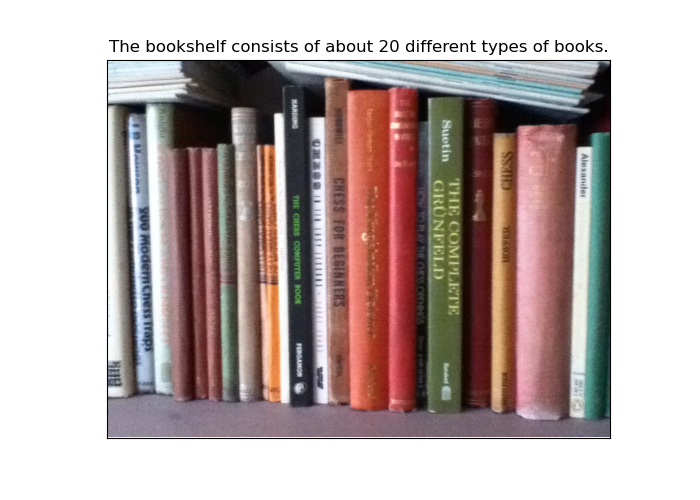}~\hfill~
				~\hfill\includegraphics[width=0.5\linewidth]{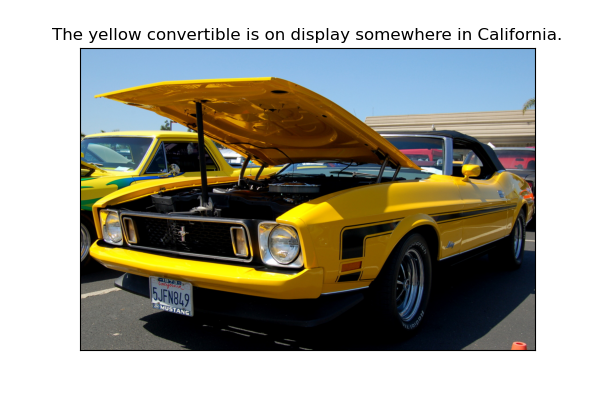}~\hfill~
	\caption{Examples of some TextCaps pairs of image/caption  from the $4$ classes we considered of (top-left) Food, (top-right) Bottle, (bottom-left) Book  (bottom-right) Car. We can see how difficult some examples can be, especially from the caption point of view since few hint about the class is provided by the text. }
	\label{fig:textcaps}
\end{figure*}

\end{document}